\documentclass[11pt]{article}
\usepackage[dvips]{graphicx}
\usepackage{amssymb,amsmath,color}
\usepackage{url}

\title{Convex Analysis and Optimization \\
 with Submodular Functions: a Tutorial}

\author{
Francis Bach  \\
INRIA - Willow project-team \\
Laboratoire d'Informatique de l'Ecole Normale Sup\'erieure \\
Paris, France  \\
\texttt{francis.bach@ens.fr} }

\newcommand{\BEAS}{\begin{eqnarray*}}
\newcommand{\EEAS}{\end{eqnarray*}}
\newcommand{\BEA}{\begin{eqnarray}}
\newcommand{\EEA}{\end{eqnarray}}
\newcommand{\BEQ}{\begin{equation}}
\newcommand{\EEQ}{\end{equation}}
\newcommand{\BIT}{\begin{itemize}}
\newcommand{\EIT}{\end{itemize}}
\newcommand{\BNUM}{\begin{enumerate}}
\newcommand{\ENUM}{\end{enumerate}}
\newcommand{\BA}{\begin{array}}
\newcommand{\EA}{\end{array}}

\newcommand{\var}{\mathop{ \rm var}}

\newcommand{\tr}{\mathop{ \rm tr}}

\newcommand{\rb}{\mathbb{R}}
\newcommand{\BlackBox}{\rule{1.5ex}{1.5ex}}  % end of proof
\newcommand{\lova}{Lov\'asz }

\newenvironment{proof}{\par\noindent{\bf Proof\ }}{\hfill\BlackBox\\[2mm]}

\newtheorem{proposition}{Proposition}
\newtheorem{definition}{Definition}

\newcommand{\mysec}[1]{Section~\ref{sec:#1}}
\newcommand{\eq}[1]{Eq.~(\ref{eq:#1})}
\newcommand{\myfig}[1]{Figure~\ref{fig:#1}}

\def \lova{Lov\'asz }

\begin{document}

\maketitle

 \section*{Introduction}
 
 Set-functions appear in many areas of computer science and applied mathematics, such as machine learning~\cite{krause2005near,kawahara22submodularity, krause2010submodular,bach2010structured}, computer vision~\cite{boykov2001fast,hochbaum2001efficient}, operations research~\cite{queyranne1995scheduling} or electrical networks~\cite{nara}. Among these set-functions, submodular functions play an important role, similar to convex functions on vector spaces. 
 In this tutorial, the theory of submodular functions is presented, in a self-contained way, with all results shown from first principles. A good knowledge of convex analysis is assumed (see, e.g.,~\cite{boyd,borwein2006caa}).
 
 Several books and tutorial articles already exist on the same topic and the material presented in this tutorial rely mostly on those~\cite{fujishige2005submodular,nara,toshev2010submodular,krause-beyond}. However, in order to present the material in the simplest way, ideas from related research papers have also been used.
 
 \paragraph{Notation.} We consider the set $V = \{1,\dots,p\}$, and its power set $2^V$, composed of the $2^p$ subsets of $V$. Given a vector $s \in \rb^p$, $s$ also denotes the modular set-function defined as $s(A) = \sum_{k \in A }s_k$. Moreover, $A \subset B$ means that $A$ is a subset of $B$, potentially equal to $B$. For $q \in [1,+\infty]$, we denote by $\| w\|_q$ the $\ell_q$-norm of $w$, by $|A|$ the cardinality of the set $A$, and, for $A \subset V = \{1,\dots,p\}$, $1_A$ denotes the indicator vector of the set $A$. If $w \in \rb^p$, and $ \alpha \in \rb$, then $\{ w \geqslant \alpha\}$ (resp.~$\{ w > \alpha\}$) denotes the subset of $V =\{1,\dots,p\}$ defined as $\{ k \in V, \ w_k \geqslant \alpha \}$ (resp.~$\{ k \in V, \ w_k > \alpha \}$). Similarly if $v \in \rb^p$, we have
 $\{ w \geqslant v \} = \{ k \in V, \ w_k \geqslant v_k \}$.

 \paragraph{Tutorial outline.} In \mysec{definitions}, we give the different definitions of submodular functions and of the associated polyhedra. In \mysec{lova}, we define the \lova extension and give its main properties. Associated polyhedra are further studied in \mysec{support}, where support functions and the associated maximizers are computed (we also detail the facial structure of such polyhedra). In \mysec{mini}, we provide some duality theory for submodular functions, while in \mysec{ope}, we present several operations that preserve submodularity. In \mysec{prox}, we consider separable optimization problems associated with the \lova extension; these are reinterpreted in \mysec{base} as separable optimization over the submodular or base polyhedra. In \mysec{sfm}, we present various approaches to submodular function minimization (without all details of algorithms). In \mysec{polym}, we specialize some of our results to non-decreasing submodular functions. Finally, in \mysec{examples}, we present classical examples of submodular functions.

 \tableofcontents

 \section{Definitions}
\label{sec:definitions}

Throughout this tutorial, we consider $V = \{1,\dots,p\}$, $p>0$ and its power set (i.e., set of all subsets) $2^V$, which is of cardinality $2^p$. We also consider a real-valued set-function $F: 2^V \to \rb$ such that $F(\varnothing)=0$. As opposed to the common convention with convex functions, we do not allow infinite values for the function $F$.

\begin{definition}[Submodular function]
\label{def:def}
A set-function $F: 2^V \to \rb$ is   submodular if and only if, for all subsets $A,B \subset V$, we have:
$F(A) + F(B) \geqslant F(A\cup B) + F(A \cap B)$.
\end{definition}

The simplest example of submodular function is the cardinality (i.e., $F(A) =|A|$ where $|A|$ is the number of elements of $A$), which is both submodular and supermodular (i.e., its opposite is submodular), which we refer to as \emph{modular}. 

From Def.~\ref{def:def}, it is clear that the set of submodular functions is closed under addition and multiplication by a positive scalar. The following proposition shows that a submodular has the ``diminishing return'' property, and that this is sufficient to be submodular. Thus, submodular functions may be seen as a discrete analog to \emph{concave} functions. However, in terms of optimization they behave more like \emph{convex} functions (e.g., efficient minimization, duality theory, linked with convex \lova extension).

\begin{proposition}[Equivalent definition with first order differences]
\label{prop:firstorder}
$F$ is submodular if and only if for all $A,B  \subset  V$ and $k \in V$, such that $A \subset B$ and $k \notin B$, we have
 $ F(A \cup \{k\}) - F(A) \geqslant  F(B \cup \{k\}) - F(B) $.
\end{proposition}
\begin{proof} Let $A \subset B$, and $k \notin B$,
$F(A \cup\{k\}) - F(A) - F(B \cup \{k\}) + F(B) = F(C)+F(D) - F(C \cup D) - F(C \cap D)$ with $C = A \cup \{k\}$ and $D = B$, which shows that the condition is necessary.  To prove the opposite, we assume that the condition is satisfied; one can first show that if $A \subset B $ and $ C \cap B = \varnothing$, then $F(A \cup C) - F(A) \geqslant F(B \cup C) - F(B)$ (this can be obtained by summing the $m$ inequalities $F(A \cup \{c_1,\dots,c_k\} ) - F(A \cup \{c_1,\dots,c_{k-1}\}) \geqslant
F(B \cup \{c_1,\dots,c_k\} ) - F(B \cup \{c_1,\dots,c_{k-1}\})$ where $C = \{c_1,\dots,c_m\}$).

Then for any $X,Y \subset V$, take $A = X \cap Y$, $C=X \backslash Y$ and $B=Y$ to obtain
$F(X) + F(Y) \geqslant F(X\cup Y) + F(X \cap Y)$, which shows that the condition is sufficient.
\end{proof}

The following proposition gives the tightest condition for submodularity (easiest to show in practice).

\begin{proposition}[Equivalent definition with second order differences]
$F$ is submodular if and only if for all $A \subset  V$ and $j,k \in V \backslash A$, we have
 $ F(A \cup \{k\}) - F(A) \geqslant  F(A \cup \{j,k\}) - F(A \cup \{ j\} ) $.
\end{proposition}
\begin{proof} 
This condition is weaker than the one from previous proposition. To prove that it is still sufficient, simply apply it to subsets $A \cup \{b_1,\dots,b_{s-1}\}$, $j=b_{s}$ for $B = A \cup \{b_1,\dots,b_m\} \supset A$ with $k \notin B$, and sum the $m$ inequalities $F(A \cup \{b_1,\dots,b_{s-1} \} \cup \{k\} ) - F( A\cup \{b_1,\dots,b_{s-1} \} \ )\geqslant F(A \cup \{b_1,\dots,b_{s} \} \cup \{k\} ) - F( A\cup \{b_1,\dots,b_{s} \}  )$, to obtain the condition in Prop.~\ref{prop:firstorder}.
\end{proof}

A vector $s \in \rb^p$ naturally leads to a modular set-function defined as $s(A) = \sum_{k \in A} s_k = s^\top 1_A$, where $1_A \in \rb^p$ is the indicator vector of the set $A$. We now define specific polyhedra in $\rb^p$.
These play a crucial role in submodular analysis, as most results may be interpreted or proved using such polyhedra.

\begin{definition}[Submodular and base polyhedra]
\label{def:polyhedra}
Let $F$ be a submodular function such that $F(\varnothing)=0$. The submodular polyhedron $P(F)$ and the base polyhedron $B(F)$ are defined as:
\BEAS
P(F) & = &  \{ s \in \rb^p, \ \forall A \subset V, s(A) \leqslant F(A) \} \\
B(F) & = &  \{ s \in \rb^p, \ s(V) = F(V), \ \forall A \subset V, s(A) \leqslant F(A) \}
=P(F) \cap \{ s(V) = F(V) \} .
\EEAS
\end{definition}

As shown in the following proposition, the submodular polyhedron $P(F)$ has non empty-interior and is unbounded. Note that the other polyhedron (the base polyhedron) will be shown to be non-empty and bounded as a consequence of Prop.~\ref{prop:greedy}. It has empty interior since it is included in the subspace $s(V)=F(V)$. See \myfig{poly} for examples with $p=2$ and $p=3$.

\begin{proposition}[Properties of submodular polyhedron]
\label{prop:nonemptyinterior}
Let $F$ be a submodular function such that $F(\varnothing)=0$. If $s \in P(F)$, then for all $t \in \rb^p$, such that $t \leqslant s$, we have $t \in P(F)$. Moreover, $P(F)$ has non-empty interior.
\end{proposition}
\begin{proof}
The first part is trivial, since $t(A) \leqslant s(A)$ if $t \leqslant s$. For the second part, we only need to show that $P(F)$ is non-empty, which is true since the constant vector equal to $\min_{ A \subset V, \ A \neq \varnothing} \frac{ F(A) }{|A|}$ belongs to $P(F)$.
\end{proof}

\begin{figure}

\begin{center}
\includegraphics[scale=.5]{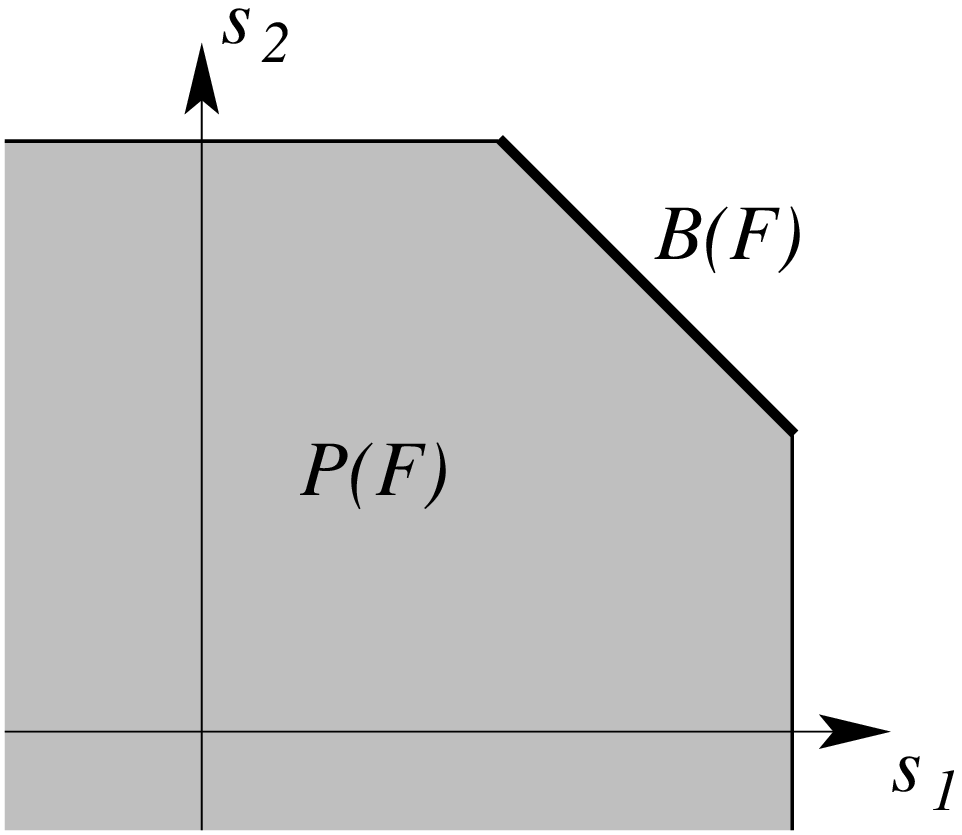}
\hspace*{1cm}
\includegraphics[scale=.5]{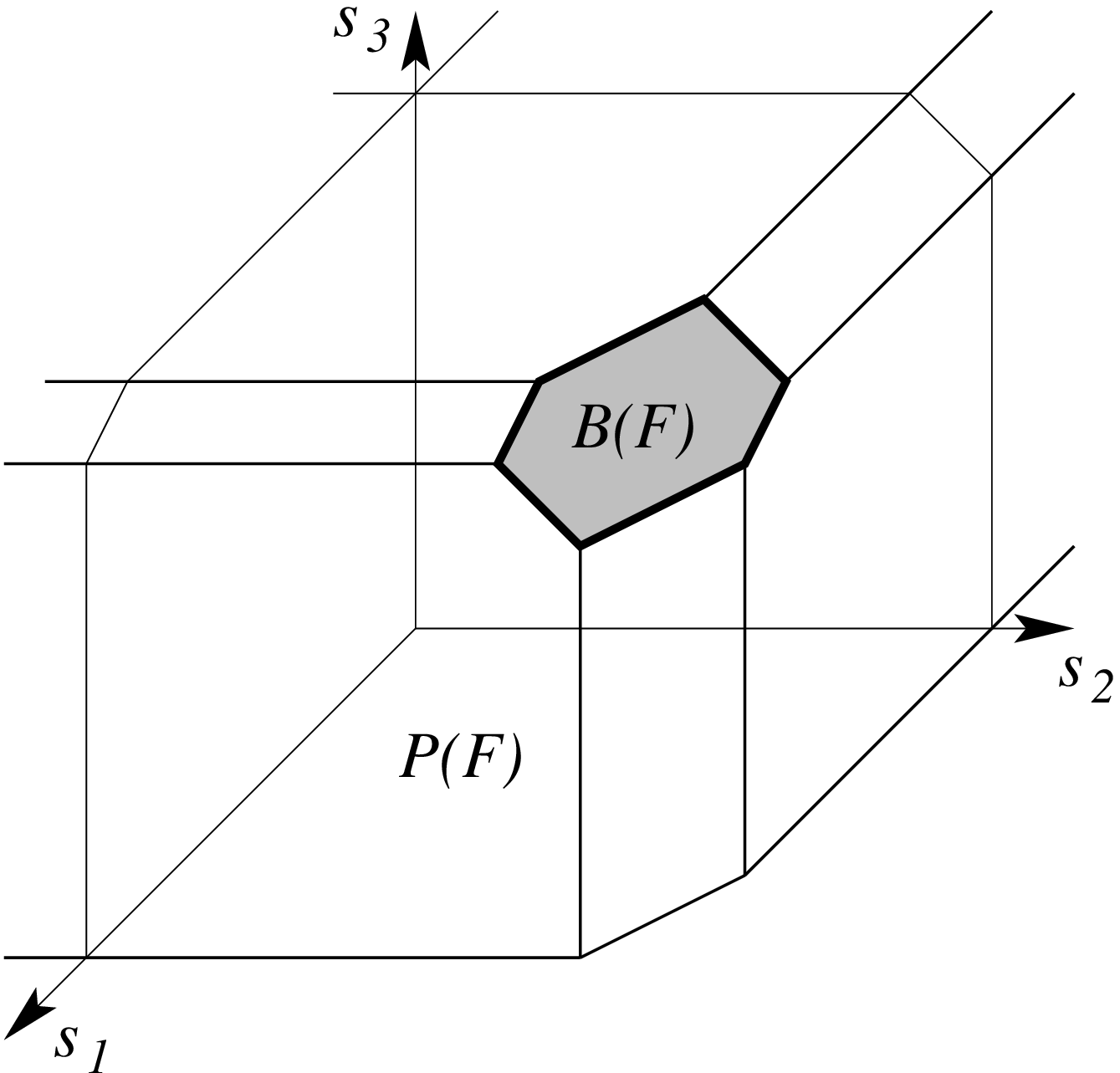}
\end{center}

\caption{Submodular polyhedron $P(F)$ and base polyhedron $B(F)$ for $p=2$ (left) and $p=3$ (right), for a non-decreasing submodular function.}
\label{fig:poly}
\end{figure}

\section{\lova extension}
\label{sec:lova}
We consider a set-function $F$ such that $F(\varnothing)=0$, \emph{which is not necessary submodular}. We can define its \lova extension~\cite{lovasz1982submodular}, which is often referred to as its Choquet integral~\cite{choquet1953theory}.
The \lova extension allows to draw links between submodular set-functions and regular convex functions, and transfer known results   from convex analysis, such as duality.

\begin{definition}[\lova extension]
Given a set-function $F$ such that $F(\varnothing)=0$, the \lova extension $f:\rb^p \to \rb$ is defined as follows; for $w \in \rb^p$, order the components $w_{j_1} \geqslant \cdots \geqslant w_{j_p}$, and define $f(w)$ through any of the following equations:
\BEA
\label{eq:lova1} f(w) &  = &  w_{j_1} F_{j_1}  + \sum_{k=2}^p w_{j_k} \big[ F(\{j_1,\dots,j_k\}) - F(\{j_1,\dots,j_{k-1}\}) \big] ,\\
\label{eq:lova2} &  = &  \sum_{k=1}^{p-1} F(\{j_1,\dots,j_k\}) (w_{j_k} - w_{j_{k+1}} ) + F(V)w_{j_p},\\
\label{eq:lova3}& = & \int_{\min \{ w_1,\dots,w_p \}}^{+\infty} F(  \{w \geqslant z\}   ) dz + F(V) \min \{ w_1,\dots,w_p \}, \\
\label{eq:lova4}& = & \int_{0}^{+\infty} F(   \{w \geqslant z\}   ) dz + 
 \int_{-\infty}^0 [ F(  \{w \geqslant z\} ) - F(V) ] dz.
 \EEA
\end{definition}
\begin{proof}
To prove that we actually define a function,
one needs to prove that the definition is independent of the non unique ordering 
$w_{j_1} \geqslant \cdots \geqslant w_{j_p}$, which is trivial from the last formulation in \eq{lova4}. The first and second formulations in \eq{lova1} and \eq{lova2} are equivalent (by integration by parts, or Abel summation formula). To show equivalence with \eq{lova3}, one may notice that 
$z \mapsto F(   \{w \geqslant z\}   ) $ is piecewise constant, with value zero for $z >  w_{j_1} = \max \{w_1,\dots,w_p\}$, and equal to $F(\{j_1,\dots,j_k\}) $ for $ z \in (w_{j_{k+1}},w_{j_k} ) $, $k=\{1,\dots,p-1\}$, and equal to $F(V)$ for $z < w_{j_p}  = \min \{ w_1,\dots,w_p \}$. What happens at break points is irrelevant for integration.

To prove \eq{lova4}, notice that for $\alpha \leqslant \min \{ 0, w_1,\dots,w_p \} $, \eq{lova3}
\BEAS
f(w)&  = &  \int_{\alpha}^{+\infty} F(   \{w \geqslant z\}   ) dz 
-  \int_{\alpha}^{\min \{ w_1,\dots,w_p \}} F(   \{w \geqslant z\}   ) dz 
 + F(V) \min \{ w_1,\dots,w_p \}  \\
 & = &   \int_{\alpha}^{+\infty} F(  \{w \geqslant z\}   ) dz 
-  \int_{\alpha}^{\min \{ w_1,\dots,w_p \}} F( V ) dz 
 +   \int_0^{ \min \{ w_1,\dots,w_p \}} F(V) dz  \\
& = &  
\int_{\alpha}^{+\infty} F(   \{w \geqslant z\}  ) dz -  \int_{\alpha}^{0} F(V) dz,
\EEAS
 and we get the result by letting $\alpha$ tend to $-\infty$.   
 \end{proof}

Note that for modular functions $A \mapsto s(A)$, with $s \in \rb^p$, then the \lova extension is the linear function $w \mapsto w^\top s$.
The following proposition details classical properties of the Choquet integral. Property~(e) below implies that the \lova extension is equal to the original set-function on $\{0,1\}^p$ (which can canonically be identified to $2^V$), and hence is indeed an \emph{extension} of $F$. 

\begin{proposition}[Properties of \lova extension]
\label{prop:lova}
Let $F$ be any set-function such that $F(\varnothing)=0$. We have:

(a) if $F$ and $G$ are set-functions with \lova extensions $f$ and $g$, then $f+g$ is the \lova extension of  $F+G$, and for all $\lambda \in \rb_+$, $\lambda f$ is the \lova extension of $\lambda F$,

(b) for $w \in \rb^p_+$,  
$f(w) = \int_{0}^{+\infty} F( \{ w \geqslant z \} ) dz$,

(c) if $F(V)=0$,   for all $w \in \rb^p$, $f(w) = \int_{-\infty}^{+\infty} F( \{ w \geqslant z \} ) dz $,

(d) for all $w \in \rb^p$ and $\alpha \in \rb$,
$f(w + \alpha 1_V) = f(w) + \alpha F(V)$,

(e) the \lova extension $f$ is positively homogeneous,

(f) for all $A \subset V$, $F(A) = f(1_A)$,

(g) if $F$ is symmetric (i.e., $\forall A \subset V, \ F(A) = F( V \backslash A)$), then $f$ is even,

(h) if $V=A_1 \cup \cdots \cup A_m$ is a partition of $V$, and $w = \sum_{i=1}^m v_i 1_{A_i}$ (i.e., is constant on each set $A_i$), with $v_1 \geqslant \cdots \geqslant v_m$, then $f(w) = \sum_{i=1}^{m-1} (v_i-v_{i+1}) F(A_1 \cup \cdots \cup A_i) + v_{i+1} F(V)$.
\end{proposition}
\begin{proof}
Properties (a), (b) and (c) are immediate from \eq{lova4} and \eq{lova2}. (d), (e) and (f) are straightforward from \eq{lova2}. If $F$ is symmetric, then $F(V)=0$, and thus
$f(-w) = \int_{-\infty}^{+\infty} F( \{ -w \geqslant z \} ) dz
\int_{-\infty}^{+\infty} F( \{ w  \leqslant -z \} ) dz =
\int_{-\infty}^{+\infty} F( \{ w \leqslant z \} ) dz
=
\int_{-\infty}^{+\infty} F( \{ w > z \} ) dz  = f(w)$ (because we may replace strict inequalities by regular inequalities), i.e., $f$ is even.
\end{proof}

Note that when the function is a cut function, then the \lova extension is related to the total variation and property (c) is often referred to as the co-area formula (see~\cite{chambolle2009total} and references therein, as well as \mysec{cuts}).

The next result relates the \lova extension with the support function of the submodular polyhedron $P(F)$ which is defined in Def.~\ref{def:polyhedra}. This is the basis for many of the theoretical results and algorithms related to submodular functions. It shows that maximizing a linear function with non-negative coefficients on the submodular polyhedron may be obtained in closed form, by the so-called ``greedy algorithm'' (see \cite{lovasz1982submodular} for an intuitive explanation), and the optimal value is equal to the value $f(w)$ of the \lova extension.
Note that otherwise, solving a linear programming problem with $2^p$ constraints would then be required.

\begin{proposition}[Greedy algorithm]
\label{prop:greedy}
Let $F$ be a submodular function such that $F(\varnothing)=0$. Let $w \in \rb_+^p$. A maximizer of $\max_{ s \in P(F) } w^\top s$ may be obtained by the following algorithm: order the components of $w$, as $w_{j_1} \geqslant \cdots \geqslant w_{j_p} \geqslant 0$ and define
$s_{j_k} =  F(\{j_1,\dots,j_k\}) - F(\{j_1,\dots,j_{k-1}\})$. Moreover,  for all $w \in \rb_+^p$, $\max_{ s \in P(F) } w^\top s = f(w)$.
\end{proposition}
\begin{proof}
	By convex duality (which applies because $P(F)$ has non empty interior from Prop.~\ref {prop:nonemptyinterior}), we have, by introducing Lagrange multipliers $\lambda_A \in \rb_+$ for the constraints $s(A) \leqslant F(A)$, $A \subset V$:
\BEAS
\max_{ s \in P(F) } w^\top s
 & = &  \min_{\lambda_A \geqslant 0, A \subset V} \max_{s \in \rb^p} \ 
 \bigg\{  w^\top s - \sum_{A \subset V} \lambda_A [ s(A) - F(A) ] \bigg\} \\
 & = &  \min_{\lambda_A \geqslant 0, A \subset V} \max_{s \in \rb^p} \ 
 \bigg\{  \sum_{A \subset V} \lambda_A F(A) + \sum_{k=1}^p s_k \big(
 w_k -   \sum_{A \ni k} \lambda_A
 \big) \bigg\} \\
 & = & \min_{\lambda_A \geqslant 0, A \subset V} \sum_{A \subset V} \lambda_A F(A)
 \mbox{ such that } \forall k \in V, \ w_k = \sum_{A \ni k} \lambda_A.
 \EEAS
 If we take the (primal) solution $s$ of the greedy algorithm, we have $f(w) = w^\top s$ from \eq{lova1}, and $s$ is feasible (i.e., in $P(F)$), because of the submodularity of $F$. Indeed, without loss of generality, we assume that $j_k=k$ for all $k \in \{1,\dots,p\}$. We can decompose $A = A_1 \cup \dots \cup A_m$, where $A_k =  (u_k,v_k]$ are \emph{integer} intervals. We then have:
\BEAS
s(A) & = & \sum_{k=1}^m \big\{  F((0,v_k]) -  F( (0,u_k]) \big\}\\
 & \leqslant & 
 \sum_{k=1}^m  \big\{ F( (u_1,v_k]) -  F( (u_1,u_k])  \big\} \mbox{ by submodularity} \\
&  = &  F((u_1,v_1]) +   \sum_{k=2}^m \big\{ F( (u_1,v_k]) -  F( (u_1,u_k])   \big\}
\\
  & \leqslant &   F((u_1,v_1]) +
 \sum_{k=2}^m   \big\{ F( (u_1,v_1] \cup (u_2,v_k]) -  F(  (u_1,v_1] \cup (u_2,u_k]) \big\}  \mbox{ by submodularity} \\
 & = &  F( (u_1,v_1] \cup (u_2,v_2]) +  \sum_{k=3}^m \big\{ F( (u_1,v_1] \cup (u_2,v_k]) -  F(  (u_1,v_1] \cup (u_2,u_k])   \big\} .
 \EEAS
 By pursuing applying submodularity, we finally obtain that
 $S(A)  \leqslant 
    F( (u_1,v_1] \cup \dots (u_m,v_m]) = F(A)$, i.e., $s \in P(F)$.

 Moreover, we can define
 dual variables $\lambda_{ \{j_1,\dots,j_k\}} = w_{j_k} - w_{j_{k+1}}$ for $k \in \{1,\dots,p-1\}$ and $\lambda_{V} = w_{j_p}$ with all other $\lambda_A$ equal to zero. Then they are all non negative (notably because $w \geqslant 0$), and satisfy the constraint  $\forall k \in V, \ w_k = \sum_{A \ni k} \lambda_A
$. Finally, the dual cost function has also value $f(w)$ (from \eq{lova2}). Thus by duality (which holds, because $P(F)$ is not empty), $s$ is an optimal solution. Note that it is not unique (see Prop.~\ref{prop:optsupport} for a description of the set of solutions).
\end{proof}

The next proposition draws precise links between convexity and submodularity, by showing that a set-function $F$ is submodular if and only if its \lova extension $f$ is convex. This is further developed in Prop.~\ref{prop:minsub} where it is shown that minimizing $F$ on $2^V$ (which is equivalent to minimizing $f$ on $\{0,1\}^p$ since $f$ is an extension of $F$) and minimizing $f$ on $[0,1]^p$ is equivalent (when $F$ is submodular).

\begin{proposition}[Convexity and submodularity]
\label{prop:convexity}
A set-function $F$ is submodular if and only if its \lova extension $f$ is convex.
\end{proposition}
\begin{proof}
Let $A,B \subset V$. The vector $1_{A \cup B} + 1_{A \cap B} = 1_A + 1_B$ has components equal to $0$ (on $V \backslash (A\cup B)$), $2$ (on $A \cap B$) and $1$ (on $A \Delta B = (A \backslash B)  \cup (B \backslash A) $). Therefore,
$f(1_{A \cup B} + 1_{A \cap B}) =
\int_{0}^2 F(1_{\{w \geqslant z\}}) dz = \int_0^1 F(A\cup B) dz +  \int_1^2 F(A\cap B) dz =
F(A \cup B) + F(A \cap B)$. 

If $f$ is convex, then by homogeneity, 
$f(1_{A  } + 1_{  B})  \leqslant  f(1_{A  })  +  f(1_{ B}) $, which is equal to 
$ F({A    })  +  F({  B})$, and thus $F$ is submodular.

If $F$ is submodular, then by Proposition~\ref{prop:greedy}, for all $w \in \rb_+^p$, $f(w)$ is a maximum of linear functions, thus, it is convex on $\rb^p_+$. Moreover, because $f(w+ \alpha 1_V) = f(w) + \alpha F(V)$, it is convex on $\rb^p$.
\end{proof}

The next proposition completes Prop.~\ref{prop:convexity} by showing that minimizing the \lova extension on $[0,1]^p$ is equivalent to minimizing it on $\{0,1\}^p$, and hence to minimizing the set-function $F$ on $2^V$  (when $F$ is submodular).

\begin{proposition}[Minimization of submodular functions]
\label{prop:minlova}
\label{prop:minsub}
\label{prop:min}
Let $F$ be a submodular function and $f$ its \lova extension; then $\min_{A \subset V}F(A)  = \min_{ w \in [0,1]^p } f(w)$.

\end{proposition}
\begin{proof} Because $f$ is an extension from $\{0,1\}^p$ to $[0,1]^p$ (property (d) from Proposition~\ref{prop:lova}), then 
we must have $\min_{A \subset V} F(A) =  \min_{ w \in \{0,1\}^p } f(w) \geqslant  \min_{ w \in [0,1]^p } f(w)$. For the other inequality, any $w \in [0,1]^p$ may be decomposed as
$w = \sum_{i=1}^p \lambda_i 1_{A_i}$ where $A_1\subset \cdots \subset A_p=V$, where $\lambda$ is nonnegative and has a sum smaller than or equal to  one (this can be obtained by considering $A_i$ the set of indices of the $i$ largest values of $w$). We then have $f(w) 
= \sum_{i=1}^p \int_{\sum_{k=1}^{i-1} \lambda_k}^{\sum_{k=1}^{i} \lambda_k} F(A_i) dz = 
\sum_{i=1}^p \lambda_i F(A_i) \geqslant \sum_{i=1}^p \lambda_i \min_{A\subset V} F(A) \geqslant \min_{A\subset V} F(A)$ (because $\min_{A\subset V} F(A) \leqslant 0$). This leads to the desired result.
\end{proof}

\section{Support function of submodular and base polyhedra}
\label{sec:support}

The next proposition completes Prop.~\ref{prop:greedy} by computing the full support function of $B(F)$ and $P(F)$ (see~\cite{boyd,borwein2006caa} for definitions of support functions), i.e., computing $\max_{s \in B(F)} w^\top s$ and $\max_{s \in P(F)} w^\top s$ for all possible $w$ (with positive and/or negative coefficients). Note the different behaviors for $B(F)$ and $P(F)$.

\begin{proposition}[Support function of submodular and base polyhedra]
\label{prop:support}
Let $F$ be a submodular function such that $F(\varnothing)=0$. We have:

 (a) for all $w \in \rb^p$, $\max_{ s \in B(F) } w^\top s = f(w)$,
 
  (b) if
$w \in \rb_+^p$,  $\max_{ s \in P(F) } w^\top s = f(w)$,

 (c) if  there exists $j$ such that $w_j<0$, then $\max_{ s \in P(F) } w^\top s = + \infty$.
\end{proposition}
\begin{proof} (a) From the proof of Prop.~\ref{prop:greedy}, for $w \in \rb_+^p$, then the result of the greedy algorithm satisfies $s(V) = F(V)$, and hence $(a)$ is true on $\rb_+^p$. For all $w$, for $\alpha $ large enough, $w + \alpha 1_V \geqslant 0$, and thus
$f(w) + \alpha F(V) = f(w+ \alpha 1_V) = \max_{ s \in B(F) } (w+ \alpha 1_V) ^\top s =  
\alpha F(V) +  \max_{ s \in B(F) } w ^\top s $, i.e., (a) is true.

Property (b) is shown in Proposition~\ref{prop:greedy}. For (c), notice that $s(\lambda) = s_0 - \lambda \delta_j \in P(F)$ for $\lambda \to + \infty$ and $s_0 \in P(F)$ and that  $w^\top s(\lambda) \to +\infty$.
\end{proof}

The next proposition shows necessary and sufficient conditions for optimality in the definition of support functions. Note that Prop.~\ref{prop:greedy} gave one example obtained from the greedy algorithm, and that we can now characterize all maximizers. Moreover, note that the maximizer is unique only when $w$ has distinct values, and otherwise, the ordering of the components of $w$ is not unique, and hence, the greedy algorithm may have multiple outputs (and all convex combinations of these are also solutions). The following proposition essentially shows what is exactly needed  to be a maximizer.

\begin{proposition}[Maximizers of the support function of submodular polyhedron]
\label{prop:optsupporttightSUB}
Let $F$ be a submodular function such that $F(\varnothing)=0$. Let $w \in (\rb_+^\ast)^p$, with unique values $v_1 > \cdots > v_m > 0 $, taken at sets $A_1,\dots,A_m$ (i.e., $V = A_1 \cup \cdots \cup A_m$ and $\forall i \in \{1,\dots,m\}, \ \forall k \in A_i, \ w_k = v_i$). Then $s$ is optimal for $\max_{s \in P(F)} w^\top s$ if and only if for all $i=1,\dots,m$,
$s(A_1 \cup \cdots \cup A_i) = F(A_1 \cup \cdots \cup A_i)$.
 \end{proposition}
\begin{proof}
Let $B_i = A_1 \cup \cdots \cup A_i$, for $i=1,\dots,m$.
From the optimization problems defined in the proof of Prop.~\ref{prop:greedy}, let $\lambda_V = v_m >0 $, and $\lambda_{B_i} = v_i - v_{i+1}>0$ for $i<m$, with all other $\lambda_A$, $A \subset V$, equal to zero. Such $\lambda$ is optimal (because the dual function is equal to $f(w)$).

Let $s \in B(F)$. We have:
\BEAS
\sum_{A \subset V} \lambda_A F(A) & = & v_m F(V) + \sum_{i=1}^{m-1} F(B_i) ( v_i - v_{i+1}) \\
 & = & v_m ( F(V) - s(V) ) + \sum_{i=1}^{m-1} [F(B_i) -s(B_i) ]( v_i - v_{i+1})   \\
 & & \hspace*{6cm} + v_m s(V) + \sum_{i=1}^{m-1} s(B_i) ( v_i - v_{i+1}) \\
& \geqslant & v_m s(V) + \sum_{i=1}^{m-1} s(B_i) ( v_i - v_{i+1}) = s^\top w.
\EEAS
Thus $s$ is optimal, if and only if the primal objective value $s^\top w$ is equal to the optimal dual objective value $\sum_{A \subset V} \lambda_A F(A) $, and thus, if and only if
there is equality in all above inequalities, hence the desired result.
 \end{proof}
 
 Note that if $v_m = 0$ in Prop~\ref{prop:optsupporttightSUB} (i.e., we take $w  \in \rb^p_+$ and there is a $w_k$ equal to zero), then the optimality condition is that for  all $i=1,\dots,m-1$,
$s(A_1 \cup \cdots \cup A_i) = F(A_1 \cup \cdots \cup A_i)$ (i.e., we don't need that $s(V)=F(V)$, i.e., the optimal solution is not necessarily in the base polyhedron).

\begin{proposition}[Maximizers of the support function of base polyhedron]
\label{prop:optsupporttight}
Let $F$ be a submodular function such that $F(\varnothing)=0$. Let $w \in \rb^p$, with unique values $v_1 > \cdots > v_m$, taken at sets $A_1,\dots,A_m$. Then $s$ is optimal for $\max_{s \in B(F)} w^\top s$ if and only if for all $i=1,\dots,m$,
$s(A_1 \cup \cdots \cup A_i) = F(A_1 \cup \cdots \cup A_i)$.
 \end{proposition}
\begin{proof} The proof follows the same arguments than for Prop.~\ref{prop:optsupporttightSUB}.
 \end{proof}

Given the last proposition, we may now give necessary and sufficient conditions for characterizing faces of the base polyhedron. We first characterize when the base polyhedron $B(F)$ has full relative interior.

\begin{definition}[Inseparable set] Let $F$ be a submodular function such that $F(\varnothing)=0$. A set $A \subset V$ is said separable if and only there is a set $B \subset A$, such that $B \neq \varnothing$, $B \neq A$ and $F(A) = F(B) + F( A \backslash A)$. If $A$ is non separable, $A$ is said inseparable.
\end{definition}

\begin{proposition}[Full-dimensional base polyhedron]
\label{prop:fulldim}
Let $F$ be a submodular function such that $F(\varnothing)=0$. 
The base polyhedron has full relative interior if and only if $V$ is not separable.
\end{proposition} 
\begin{proof}
If $V$ is separable into $A$ and $V \backslash A$, then for all $s \in B(F)$, we must have $s(A) = F(A)$ and hence the base polyhedron is included in the intersection of two affine hyperplanes, i.e., $B(F)$ does not have full relative interior in $\{ s(V)=F(V) \}$.

We now assume that $B(F)$ is included in $\{s(A) = F(A)\}$, for $A$ as a non-empty strict subset of $V$. Then $B(F)$ can be factorized in to $B(F_A) \times B(F^A)$ where $F_A$ is the restriction of $F$ to $A$ and $F^A$ the contraction of $F$ on $A$. 
Indeed, if $s \in B(F)$, then $s_A \in B(F_A)$ because $s(A)=F(A)$, and $s_{V \backslash A} \in B(F^A)$, because for $B \subset V \backslash A$, $s_{V \backslash A}(B) = s(B)  = s( A \cup B) - s(A) \leqslant F( A \cup B) - F(A)$. Similarly, if $s \in B(F_A) \times B(F^A)$, then for all set $B \subset V$, $s(B) = s( A \cap B) + S( (V \backslash A) \cap B) \leqslant F(A \cap B) + F( A \cup B) - F(A) \leqslant F(B)$ by submodularity, and $s(A) = F(A)$.

This shows that $f(w) =  f_A(w_A) + f^A(w_{V \backslash A} ) $, which implies that $F(V) = F(A)+F(V \backslash A)$, when applied to $w = 1_V$, i.e., $V$ is separable.
\end{proof}

We can now detail the facial structure of the base polyhedron, which will be dual to the one of the polyhedron defined by $\{ w \in \rb^p, \ f(w) \leqslant 1\}$ (i.e., level set of the \lova extension). As the base polyhedron $B(F)$ is a polytope in dimension $p-1$ (because it is bounded and contained in the affine hyperplane $\{ s(V) = F(V) \}$), one can define a set of \emph{faces}. Faces are the intersections of the polyhedron $B(F)$ with any of its supporting hyperplanes. Supporting hyperplanes are themselves defined as the hyperplanes $w^\top s = \max_{s \in B(F)} w^\top s = f(w)$ for $w \in \rb^p$.
From Prop.~\ref{prop:optsupporttight}, faces (which potentially empty relative interior) are obtained as the intersection of $B(F)$ with $s(A_1 \cup \cdots \cup A_i) = F( A_1 \cup \cdots  \cup A_i)$ for an ordered partition of $V$. Together with Prop.~\ref{prop:fulldim}, we can now provide characterization of the faces of $B(F)$.

\begin{proposition}[Faces of the base polyhedron]
\label{prop:faces}
Let $A_1 \cup \cdots \cup A_m$ be an ordered partition of $V$, such that for all $j \in \{1,\dots,m\}$, $A_j$ is inseparable for the function $G_j: B \mapsto F( A_1 \cup \cdots  \cup A_{j-1} \cup B) - F( A_1 \cup \cdots  \cup A_{j-1})$ defined on subsets of $A_j$, then  the set of bases $s \in B(F)$ such that for all $j \in \{1,\dots,m\}$, $s(A_1 \cup \cdots \cup A_i) = F( A_1 \cup \cdots  \cup A_i)$ is a proper face of $B(F)$ with non-empty relative interior.

\end{proposition}

\begin{proof}
We have a face from Prop.~\ref{prop:optsupporttight}, and it has non empty interior by applying
Prop.~\ref{prop:fulldim} on each submodular function $G_j$.
\end{proof}

\vspace*{.5cm}

The next proposition computes the Fenchel conjugate of the \lova extensions restricted to $[0,1]^p$, noting that by Prop.~\ref{prop:support}, the regular Fenchel conjugate of the unrestricted \lova extension is the indicator function of the base polyhedron (for a definition of Fenchel conjugates, see~\cite{boyd,borwein2006caa}). This allows a form of conjugacy between set-functions and convex functions.

\begin{proposition}[Conjugate of a submodular function]
\label{prop:conj}
Let $F$ be a submodular function such that $F(\varnothing)=0$. The conjugate $\tilde{f}: \rb^p \to \rb$ of $F$ is defined as $\tilde{f}(s) = \max_{A\subset V} s(A) - F(A)$. Then, the conjugate function $\tilde{f}$ is convex, and is equal to the Fenchel-conjugate of the \lova extension restricted to $[0,1]^p$. Moreover, for all $A \subset V$, $F(A) = \max_{s \in \rb^p} s(A) - \tilde{f}(s)$.
\end{proposition}
\begin{proof}
The function $\tilde{f}$ is a maximum of linear functions and thus it is convex. We have for $s \in \rb^p$:
$$
\max_{ w \in [0,1]^p} w^\top s - f(w) = \max_{A \subset V} s(A) - F(A) = \tilde{f}(s)
$$
because $F-s$ is submodular and because of Proposition~\ref{prop:minlova}, which leads to first the desired result. The last assertion is a direct consequence of the fact that $F(A) = f(1_A)$.
\end{proof}

\section{Minimizers of submodular functions}
\label{sec:mini}

In this section, we review some relevant results for submodular function minimization (for which algorithms are presented in \mysec{sfm}).

\begin{proposition}[Property of minimizers of submodular functions]
\label{prop:minimizer}
\label{prop:mincond}
Let $F$ be a submodular function such that $F(\varnothing)=0$. The set $A \subset V$ is a minimizer of $F$ on $2^V$ if and only if $A$ is a minimizer of the function from $2^A$ to $\rb$ defined as $B \subset A \mapsto F(B)$, and if $\varnothing$ is a minimizer of the function from $2^{V\backslash A}$ to $\rb$ defined as $B \subset V\backslash A \mapsto F(B \cup A) - F(A)$.
 \end{proposition}
\begin{proof}
 The set of two conditions is clearly necessary. To show that it is sufficient, we let $B \subset V$,
 we have:
 $F(A) + F(B) \geqslant F(A \cup B) + F(A \cap B) \geqslant F(A) + F(A)$, by using the submodularity of $F$ and then the set of two conditions. This implies that $F(A) \leqslant F(B)$, for all $B \subset V$, hence the desired result.
\end{proof}

The following proposition provides a useful step towards submodular function minimization. In fact, it is the starting point of most polynomial-time algorithms presented in \mysec{sfm}.

\begin{proposition}[Dual of minimization of submodular functions]
\label{prop:dualmin}
Let $F$ be a submodular function such that $F(\varnothing)=0$.
We have: 
\BEQ \min_{A\subset V} F(A) = \max_{ s \in B(F) } s_-(V), \EEQ
where $s_- = \min\{s,0\}$. Moreover, given $A \subset V$ and $s\in B(F)$, we always have $F(A) \geqslant s_-(V)$ with equality if and only if
$\{ s < 0 \} \subset A \subset \{ s \leqslant 0\}$ and $A$ is tight for $s$, i.e., $s(A) = F(A)$.

We also have
\BEQ
 \min_{A\subset V} F(A) =  \max_{ s \in P(F), \ s \leqslant 0 } s(V). 
 \EEQ
Moreover,  given $A \subset V$ and $s\in P(F) $ such that $s \leqslant 0$, we always have $F(A) \geqslant s(V)$ with equality if and only if
$\{ s < 0 \} \subset A $ and $A$ is tight for $s$, i.e., $s(A) = F(A)$.
\end{proposition}
\begin{proof}
We have, by convex duality, and Props.~\ref{prop:minlova} and \ref{prop:support}:
$$
\min_{A\subset V} F(A)  =  \min_{ w \in [0,1]^p } f(w) =  \min_{ w \in [0,1]^p } \max_{ s\in B(F)} w^\top s = 
\max_{ s\in B(F)} \min_{ w \in [0,1]^p }  w^\top s =  
\max_{ s\in B(F)}  s_-(V). $$
Strong duality indeed holds because of Slater's condition ($[0,1]^p$ has non empty interior). 
Moreover, we have, for all $A \subset V$ and $s\in B(F)$:
$$F(A) \geqslant s(A) = s(A \cap \{ s < 0 \} ) + s(A \cap \{ s > 0 \} )   \geqslant  s(A \cap \{ s < 0 \} )  \geqslant s_-(V)$$
with equality if there is equality in the three inequalities. The first one leads to $s(A) =F(A)$. The second one leads to $A \cap \{ s > 0 \} = \varnothing$, and the last one leads to  $\{ s < 0 \} \subset A $.
Moreover,
\BEAS
 \max_{ s \in P(F), \ s \leqslant 0 } s(V)
 & = &  \max_{ s \in P(F)} \min_{ w \geqslant 0 } s^\top 1_V - w^\top s
 =  \min_{ w \geqslant 0 }  \max_{ s \in P(F)} s^\top 1_V - w^\top s \\
&  =  &  \min_{  1 \geqslant w \geqslant 0 } f(1_V-w) \mbox{ because of property (c) in Prop.~\ref{prop:support}}\\
&  = &  \min_{A\subset V} F(A) \mbox{ because of Prop.~\ref{prop:min}} .
 \EEAS
Moreover, given $s \in P(F)$ such that $s \leqslant 0$ and $A \subset V$, we have:
$$
F(A) \geqslant s(A) =  s(A \cap \{ s < 0 \} )  \geqslant s(V)
$$
with equality if and only if $A$ is tight and $ \{ s< 0 \} \subset A$.
\end{proof}

\section{Operations that preserve submodularity}
\label{sec:ope}

In this section, we present several ways of building submodular functions from existing ones. For all of these, we describe how the \lova extensions and the submodular polyhedra are affected. Note that in many cases, operations are simpler in terms of polyhedra.

\begin{proposition}[Restriction of a submodular function]
\label{prop:restriction}
let $F$ be a submodular function such that $F(\varnothing)=0$ and $A \subset V$. The restriction of $F$ on $A$, denoted $F_A$ is a set-function on $A$ defined as $F_A(B) = F(B)$ for $B \subset A$. The function $f_A$ is submodular. Moreover, if we can write the \lova extension of $F$ as $f(w) = f( w_A , w_{V \backslash A})$, then the \lova extension of $F_A$ is $f_A(w_A) = f(w_A,0)$. Moreover, the submodular polyhedron $P(F_A)$ is simply the projection of $P(F)$ on the components indexed by $A$, i.e., $s \in P(F_A)$ if and only if $\exists t $ such that $(s,t) \in P(F)$.
 \end{proposition}
\begin{proof} Submodularity and the form of the \lova extension  are straightforward from definitions. To obtain the submodular polyhderon, notice that we have
$f_A(w_A) = f(w_A,0)  = \max_{(s,t) \in P(F)} w_A^\top s + 0^\top t$, which implies the desired result, this shows that the Fenchel-conjugate of the \lova extensions is the indicator function of a polyhedron.
\end{proof}

\begin{proposition}[Contraction of a submodular function]
\label{prop:contraction}
let $F$ be a submodular function such that $F(\varnothing)=0$ and $A \subset V$. The contraction of $F$ on $A$, denoted $F^A$ is a set-function on $V \backslash A$ defined as $F^A(B) = F(A \cup B) - F(A)$ for $B \subset V \backslash A$. The function $F^A$ is submodular. Moreover, if we can write the \lova extension of $F$ as $f(w) = f( w_A , w_{V \backslash A})$, then the \lova extension of $F^A$ is $f^A(w_{V \backslash A}) = f(1_A,w_{V \backslash A}) -  F(A)$. Moreover, the submodular polyhedron $P(F^A)$ is simply the projection of $P(F) \cap \{ s(A) = F(A) \}$ on the components indexed by $V \backslash A$, i.e., $t \in P(F^A)$ if and only if $\exists s \in P(F) \cap \{ s(A) = F(A) \}$, such that $s_{V \backslash A} =t$.
 \end{proposition}
\begin{proof} Submodularity and the form of the \lova extension  are straightforward from definitions. Let $t \in \rb^{| V \backslash A|}$. If  $\exists s \in P(F) \cap \{ s(A) = F(A) \}$, such that $s_{V \backslash A} =t$, then we have for all $B \subset V \backslash A$, $t(B) = t(B) + s(A) - F(A) \leqslant F(A \cup B) - F(A)$, and hence $t \in P(F^A)$. If $t \in P(F^A)$, then take any $v \in B(F_A)$ and concatenate $v$ and $t$ into $s$. Then, for all subsets $C \subset V$,
$s(C) = s(C \cap A) + s( C \cap ( V \backslash A) ) = 
 v(C \cap A) + t( C \cap ( V \backslash A) ) \leqslant F(C \cap A) + F( A \cup ( C \cap ( V \backslash A)  ) ) - F(A)
 = F(C \cap A) + F( A    \cup C ) - F(A) \leqslant F(C)$ by submodularity. Hence $s \in P(F)$.

\end{proof}

The next proposition shows how to build a new submodular function from an existing one, by partial minimization. Note the similarity (and the difference) between the submodular polyhedra for a partial minimum (Prop.~\ref{prop:partial}) and for the restriction defined in Prop.~\ref{prop:restriction}.

\begin{proposition}[Partial minimum of a submodular function]
\label{prop:partial}
We consider a submodular function $G$ on $V \cup W$, where $V \cap W = \varnothing$ (and $|W|=q$), with \lova extension $g:\rb^{p+q} \to \rb$. We consider, for $A \subset V$, $F(A) = \min_{B \subset W} G(A \cup B) -    \min_{B \subset W} G(  B)  $. The set-function $F$ is submodular and such that $F(\varnothing)=0$. Its \lova extension is such that for all $w \in [0,1]^p$, $f(w) = \min_{ v\in [0,1]^q} g(w,v) - \min_{ v\in [0,1]^q} g(0,v) $. 
Moreover, if $ \min_{B \subset W} G( B) =0$, we have for all $w \in \rb_+^p$,    $f(w) = \min_{ v\in \rb_+^q} g(w,v)  $, and the submodular polyhedron $P(F) $ is the set of $s \in \rb^p$ such that there exists $t \in \rb_+^q$, such that $(s,t) \in P(G)$.
 \end{proposition}
\begin{proof} Define $c = \min_{B \subset W} G( B) $, which is independent of $A$. We have, for $A,A' \subset V$, and any $B,B' \subset W$, by definition of $F$:
\BEAS
F(A \cup A') + F( A\cap A')
& \leqslant & -2c +  G([ A \cup A'] \cup [ B' \cup B' ] ) + G( [ A \cap A' ] \cup [ B' \cap B' ] )  \\
& = &  -2c + G([ A \cup B ] \cup [ A' \cup B' ] ) + G( [ A \cup B ] \cap [ A' \cup B' ] )  \\
& \leqslant &  -2c + G( A \cup B ) + G ( A' \cup B'  )
\mbox{ by submodularity}
.
\EEAS
Minimizing with respect to $B$ and $B'$ leads to the submodularity of $F$.

Following Prop.~\ref{prop:conj}, we can get the conjugate function $\tilde{f}$ from the one $\tilde{g}$ of $G$.
For $s \in \rb^p$, we have, by definition, $\tilde{f}(s) = \max_{A\subset V} s(A) - F(A)
=\max_{A \cup B \subset V \cup W } s(A) +c - G(A\cup B)  = c + \tilde{g}(s,0)$. We thus get from  Prop.~\ref{prop:conj} that for $w\in [0,1]^p$,
\BEAS
f(w)&  = &  \max_{s \in \rb^p} w^\top s - \tilde{f}(s) \\
&  = &  \max_{s \in \rb^p} w^\top s - \tilde{g}(s,0) - c \\
 &  = &  \max_{s \in \rb^p} \min_{(\tilde{w},v) \in [0,1]^{p+q}}w^\top s - \tilde{w}^\top s + g(\tilde{w},v)  - c  \mbox{ by applying  Prop.~\ref{prop:conj},}\\
 &  = &    \min_{(\tilde{w},v) \in [0,1]^{p+q}}  \max_{s \in \rb^p} w^\top s - \tilde{w}^\top s + g(\tilde{w},v) - c  \\
 &  = &    \min_{v\in [0,1]^{q}}   g(w,v)   - c \mbox{ by maximizing with respect to }s .
\EEAS
Note that $c = \min_{B \subset W} G(  B) = \min_{v \in [0,1]^q} g(0,v)$.

For any $w \in \rb_+^p$, for any $\lambda \geqslant \|w\|_\infty$, we have $w/\lambda \in [0,1]^p$, and thus
\BEAS
f(w) & = & \lambda f(w/\lambda) = \min_{v \in [0,1]^q} \lambda g(w/\lambda ,v ) - c \lambda
=\min_{v \in [0,1]^q} g(w, \lambda v ) - c \lambda
\\
 & = &  \min_{v \in [0,\lambda]^q}   g(w,v ) - c \lambda.
\EEAS
Thus, if $c=0$, we have $f(w) =  \min_{v \in \rb_+^q} g(w,v)$, by letting $\lambda \to + \infty$.
We then also have:
\BEAS
f(w) & = & \min_{v \in \rb_+^q} g(w,v)
 = \min_{v \in \rb_+^q} \max_{ (s,t) \in P(G)} w^\top s  + v^\top t \\
 & = &  \max_{ (s,t) \in P(G), \  t\in \rb^q_+} w^\top s.
\EEAS
 \end{proof}

The following propositions give an interpretation of the intersection between the submodular polyhedron and sets of the form $\{s \leqslant z\}$ and $\{s \geqslant z\}$.

\begin{proposition}[Convolution of a submodular function and a modular function]
\label{prop:conv}
Let $F$ be a submodular function such that $F(\varnothing)=0$ and $z \in \rb^p$. Define
$G(A) = \min_{ B \subset A} F(B) + z(A \backslash B)$. Then $G$ is submodular and the submodular polyhedron $P(G)$ is equal to $P(F) \cap \{ s \leqslant z \}$. Moreover, for all $A\subset V$,  $G(A) \leqslant F(A)$ and $G(A)\leqslant z(A)$.
\end{proposition}
\begin{proof} Let $A,A' \subset V$, and $B,B'$ the corresponding minimizers defining $G(A)$ and $G(A')$. We have:
\BEAS
G(A) + G(A') & = & F(B) + z(A \backslash B) +  F(B') + z(A' \backslash B') \\
& \geqslant & F(B \cup B') + F(B \cap B') +  z(A \backslash B) + z(A' \backslash B')  \mbox{ by submodularity} \\
& = & F(B \cup B') + F(B \cap B') +  z([A \cup A'] \backslash [B \cup B'])
+  z([A \cap A'] \backslash [B \cap B']) \\
& \geqslant & G(A \cup A') + G(A \cap A') \mbox{ by definition of } G,
\EEAS
hence the submodularity of $G$. If $s \in P(G)$, then  $\forall B \subset A \subset V$, $s(A) \leqslant 
G(A) \leqslant  F(B) + z(A \backslash B)$. From $B=A$, we get that $s \in P(F)$; from $B = \varnothing$, we get $s \leqslant z$, and hence $s \in P(F) \cap \{ s \leqslant z \}$. If $s \in P(F) \cap \{ s \leqslant z \}$, for all $\forall B \subset A \subset V$,
$s(A) = s(A \backslash B) + s(B) \leqslant  z(A \backslash B) + F(B)$; by minimizing with respect to $B$, we get that $s \in P(G)$.

We get $G(A) \leqslant F(A)$ by taking $B=A$ in the definition of $G(A)$, and we get $G(A) \leqslant z(A)$ by taking $B  = \varnothing$.
\end{proof}

\begin{proposition}[Monotonization of a submodular function]
\label{prop:monotone}
Let $F$ be a submodular function such that $F(\varnothing)=0$. Define
$G(A) = \min_{ B \supset A} F(B) - \min_{ B \subset V} F(B) $. Then $G$ is submodular such that $G(\varnothing)=0$, and the base polyhedron $B(G)$ is equal to $B(F) \cap \{ s \geqslant 0 \}$. Moreover, $G$ is non-decreasing, and for all $A\subset V$,  $G(A) \leqslant F(A)$.
\end{proposition}
\begin{proof} Let $c = \min_{ B \subset V} F(B) $. Let $A,A' \subset V$, and $B,B'$ the corresponding minimizers defining $G(A)$ and $G(A')$. We have:
\BEAS
G(A) + G(A') & = & F(B)   +  F(B')   - 2c  \\
& \geqslant & F(B \cup B') + F(B \cap B')  - 2c  \mbox{ by submodularity} \\
& \geqslant & G(A  \cup A') + G(A \cap A')  \mbox{ by definition of } G,
\EEAS
hence the submodularity of $G$. It is obviously non-decreasing. We get $G(A) \leqslant F(A)$ by taking $B=A$ in the definition of $G(A)$. Since $G$ is increasing, $B(G) \subset \rb_+^p$ (because all of its extreme points, obtained by the greedy algorithm, are in $\rb_+^p$). By definition of $G$, $B(G) \subset B(F)$. Thus $B(G) \subset B(F) \cap \rb_+^p$. The opposite inclusion is trivial from the definition.

\end{proof}

\section{Proximal optimization problems}
\label{sec:prox}

In this section, we consider separable convex functions and   the minimization of such functions penalized by the \lova extension of a submodular function. When the separable functions are all quadratic functions, those problems are often referred to as \emph{proximal problems} (see, e.g.,~\cite{combettes2010proximal} and references therein). We make the simplifying assumption that the problem is strictly convex and differentiable (but not necessarily quadratic), but sharp statements could also be made in the general case. The next proposition shows that it is equivalent to the maximization of a separable concave function over the base polyhedron.

\begin{proposition}[Dual of proximal optimization problem]
\label{prop:prox}
Let $\psi_1,\dots,\psi_p$ be  $p$ continuously differentiable strictly convex functions on $\rb$, with Fenchel-conjugates $\psi_1^\ast, \dots,\psi^\ast_p$. The two following optimization problems are dual of  each other:
 \BEQ
\label{eq:prox}
 \min_{w \in \rb^p} f(w) + \sum_{j=1}^p \psi_j(w_j),
 \EEQ
 \BEQ
 \label{eq:proxdual} 
   \max_{s \in B(F)}   - \sum_{j=1}^p \psi_j^\ast(-s_j) .
 \EEQ
 The pair $(w,s)$ is optimal if and only if $s_k = - \psi_k'(w_k)$ for all $k \in \{1,\dots,p\}$, and $s \in B(F)$ is optimal for the maximization of $w^\top s$ over $s \in B(F)$ (see Prop.~\ref{prop:optsupporttight} for optimality conditions).
\end{proposition}
\begin{proof}
We have:
\BEAS
\nonumber \min_{w \in \rb^p} f(w) + \sum_{j=1}^p \psi_i(w_j)
\nonumber &  = &  \min_{w \in \rb^p}  \max_{s \in B(F)} w^\top s  + \sum_{j=1}^p \psi_j(w_j) \\
\nonumber &  = &  \max_{s \in B(F)}   \min_{w \in \rb^p}  w^\top s  + \sum_{j=1}^p \psi_j(w_j) \\
 &  = &  \max_{s \in B(F)}   - \sum_{j=1}^p \psi_j^\ast(-s_j) ,
\EEAS
where $\psi_j^\ast$ is the Fenchel-conjugate of $\psi_j$ (which may in general have a domain strictly included in $\rb$).
Thus the separably penalized problem defined in \eq{prox} is equivalent to a separable maximization over the base polyhedron (i.e., \eq{proxdual}). Moreover, the unique optimal $s$ for \eq{proxdual} and the unique optimal $w$ for \eq{prox} are related through $s_j = - \psi_j'(w_j)$ for all $j \in V$.
\end{proof}
 
For simplicity, we now assume that for all $j \in V$,  functions $\psi_j$ are such that $\sup_{\alpha \in \rb} \psi_j'(\alpha) = +\infty$
 and $\inf_{\alpha \in \rb} \psi_j'(\alpha) = -\infty$.  This implies that the Fenchel-conjugates $\psi_j^\ast$ are defined and finite on $\rb$. Following~\cite{chambolle2009total}, we also consider a sequence of set optimization problems, parameterized by $\alpha \in \rb$:
\BEQ
\label{eq:proxalpha}
 \min_{A \subset V} F(A) + \sum_{j \in A}  \psi_j'(\alpha)
 \EEQ
We denote by $A^\alpha$ any minimizer of \eq{proxalpha}. Note that $A^\alpha$ is a minimizer of a submodular function $F + \psi'(\alpha)$, where $\psi'(\alpha) \in \rb^p$ is the vector of components $\psi_k'(\alpha)$.

The main property, as shown in~\cite{chambolle2009total}, is that solving \eq{prox}, which is a convex optimization problem, is equivalent to solving \eq{proxalpha} for all possible $\alpha$, which are submodular optimization problems. We first show a monotonicity property of solutions of \eq{proxalpha}.

\begin{proposition}[Monotonicity of solutions]
\label{prop:monotonicity}
If $\alpha > \beta$, then  any solutions $A^\alpha$ and $A^\beta$ of \eq{proxalpha} for $\alpha$ and $\beta$ satisfy $A^\alpha \subset A^\beta$.
\end{proposition}
\begin{proof}
We have, by optimality of $A^\alpha$ and $A^\beta$:
\BEAS
F(A^\alpha) + \sum_{j \in A^\alpha}  \psi_j'(\alpha) 
 & \leqslant &   F(A^\alpha \cup A^\beta) + \sum_{j \in A^\alpha \cup A^\beta}  \psi_j'(\alpha) \\
F(A^\beta) + \sum_{j \in A^\beta}  \psi_j'(\beta) 
 & \leqslant &   F(A^\alpha \cap A^\beta) + \sum_{j \in A^\alpha \cap A^\beta}  \psi_j'(\beta) , 
\EEAS
and by summing the two inequalities and using the submodularity of $F$, 
$$
 \sum_{j \in A^\alpha}  \psi_j'(\alpha) + \sum_{j \in A^\beta}  \psi_j'(\beta) 
 \leqslant  \sum_{j \in A^\alpha \cup A^\beta}  \psi_j'(\alpha)
 + \sum_{j \in A^\alpha \cap A^\beta}  \psi_j'(\beta),
$$
which is equivalent to
$\sum_{j \in  A^\alpha \backslash A^\beta}  (  \psi_j'(\beta) - \psi_j'(\alpha)  ) \geqslant 0$, which implies, since for all $j \in V$, $\psi_j'(\beta)<\psi_j'(\alpha)$ (because of strict convexity), that $A^\alpha \backslash A^\beta = \varnothing$.
\end{proof}

The next proposition shows that we can obtain the unique solution of \eq{prox} from all solutions of \eq{proxalpha}.

\begin{proposition}[Proximal problem from submodular function minimizations]
\label{prop:proxmin}
Given any solutions $A^\alpha$ of problems in \eq{proxalpha}, for all $\alpha \in \rb$, we define the vector $u \in \rb^p$ as
$$
u_j = \sup( \{
\alpha \in \rb, \ j \in A^\alpha
\}).
$$
Then $u$ is the unique solution of the proximal problem in \eq{prox}.
\end{proposition}
\begin{proof}
Because $\inf_{\alpha \in \rb} \psi_j'(\alpha) = -\infty$, for $\alpha$ small enough, we must have $A^\alpha=V$, and thus $u_j$ is well-defined and finite for all $j \in V$.

If $ \alpha >  u_j $, then, by definition of $u_j$,  $j \notin A^\alpha$. This implies
that $A^\alpha \subset \{ j \in V, u_j \geqslant \alpha \}  = \{ u \geqslant \alpha \}$.
Moreover, if $u_j > \alpha$, there exists $\beta \in ( \alpha, u_j) $ such that $j \in A^\beta$. By the monotonicity property of Prop.~\ref{prop:monotonicity}, $A^\beta$ is included in $A^\alpha$. This implies 
$\{ u >\alpha \} \subset A^\alpha$.

We have for all $w \in \rb^p$, and $ \beta $ less than the smallest of $(w_j)_-$ and the smallest of $(u_j)_-$ :
\BEAS
& & f(u)  + \sum_{j=1}^p \psi_j(u_j) \\
& = & 
\int_{0}^\infty F( \{u \geqslant \alpha\} ) d\alpha
+ \int_{\beta }^0 ( F( \{u \geqslant \alpha\} )  - F(V) ) d\alpha
+ \sum_{j=1}^p \bigg\{ \int_{\beta}^{u_j}  \psi_j'(\alpha) d \alpha + \psi_j(\beta) \bigg\} \\
& = &   C + 
\int_{\beta}^\infty
\bigg[
 F( \{u \geqslant \alpha\} ) + \sum_{j=1}^p  1_{u \geqslant \alpha} \psi_j'(\alpha)
\bigg]
 d\alpha  \mbox{ with } C = \int_0^\beta F(V) d\alpha + \sum_{j=1}^p \psi_j(\beta) \\
& \leqslant &   C + 
\int_{\beta}^\infty
\bigg[
 F( \{w \geqslant \alpha\} ) + \sum_{j=1}^p  1_{w \geqslant \alpha} \psi_j'(\alpha)
\bigg]
 d\alpha \mbox{ by optimality of } A^\alpha\\
& = & f(w)  + \sum_{j=1}^p \psi_j(w_j).
\EEAS
This shows that $u$ is the unique optimum of problem in \eq{prox}.
\end{proof}

From the previous proposition, we also get the following corollary, i.e., all solutions of \eq{proxalpha} may obtained from the single solutions of \eq{prox}.

\begin{proposition}[Submodular function minimizations from proximal problem]
\label{prop:subfromprox}
If $u$ is the unique minimizer of \eq{prox}, then for all $\alpha \in \rb$, the minimal minimizer of \eq{proxalpha} is ${ u > \alpha }$ and the maximal minimizer is $\{ u \geqslant \alpha\}$, that is, the minimizers $A^\alpha$ are the sets such that
$\{ u > \alpha\} \subset A^\alpha \subset \{ u \geqslant \alpha \}$.
\end{proposition}

Given the previous propositions, we can solve a sequence of problems in \eq{proxalpha}, with decreasing $\alpha$'s, in order to obtain the unique minimizer $w$ of \eq{prox}. Note that because of the monotonicity, the sets $A^\alpha$ can only increase. When a certain $j \in V$ enters $A^\alpha$, then $w_j$ is exactly equal to the corresponding $\alpha$. Once we know the largest values of $w$, we may redefine the problem by restricting on the unknown indices of $w$, which is valid for smaller values of $\alpha$.

\section{Optimization over the base polyhedron}
\label{sec:base}

Optimization of separable functions over the base polyhedron has many applications, e.g., minimization of a submodular function (from Prop.~\ref{prop:dualmin}), proximal methods described in \mysec{prox} (e.g., Prop~\ref{prop:prox}). In this section, we study these problems in more details.

\subsection{Optimality conditions}
We first show that when optimizing on the base polyhedron $B(F)$, then one only needs to look at directions of the form $\delta_k - \delta_q$ for certain pairs $(k,q)$, which will be said \emph{exchangeable} ($\delta_k \in \rb^p$ is the vector which is entirely equal to zero, except a component equal to one at position $k$, which can also denote $1_{\{k\}}$).

\begin{definition}[Tight sets]
Given a base $s \in B(F)$, a set $A \subset V$ is said tight if $s(A)=F(A)$. 
\end{definition}

\begin{proposition}[Lattice of tight sets]
\label{prop:lattice}
If $A$ and $B$ are tight for $s \in B(F)$, then $A \cap B$ and $A \cup B$ are also tight for $s$.
\end{proposition}
\begin{proof}
We have:
$$ F(A \cup B) + F( A\cap B) \geqslant s(A \cup B) + s(A \cap B)
= s(A)+s(B) = F(A) + F(B) \geqslant  F(A \cup B) + F( A\cap B).$$
Thus there is equality everywhere, which leads to the desired result. Note that this shows that the set of tight sets for $s \in \rb^p$ is a lattice.
 
\end{proof}

 We now define the notion of exchangeable pairs, which we allow us to describe the tangent cone of the base polyhedron in Prop.~\ref{prop:tangent}.
 
\begin{definition}[Dependence function and exchangeable pairs]
Given a base $s \in B(F)$ and $k \in A$, the dependence function ${ \rm Dep}(s,k)$ is the (non-empty) smallest tight set that contains $k$. If $g \in { \rm Dep}(s,k)$, then the pair $(k,g)$ is said exchangeable.
\end{definition}

Prop.~\ref{prop:lattice} shows that ${ \rm Dep}(s,k)$ is indeed well-defined because $V$ is tight and contains $k$, and the set of tight sets containing $k$ is a lattice. The following proposition details the most important properties of exchangeable pairs, which are straightforward given the definition (in fact, the conjunction of these two properties is equivalent to the definition of exchangeable pairs).

\begin{proposition}[Properties of exchangeable pairs]
\label{prop:pairs}
Let $s \in B(F)$ and $(k,q)$ is an exchangeable pair for $s$. Then:

(a)  there exists $A\subset V$ such that $k,q \in A$ and $A$ is tight for $s$,

(b) if $A \subset V$ is tight for $s$,  then $ k \in A \Rightarrow q \in A$.
 \end{proposition}
 
 The next proposition shows that only these exchangeable pairs need to be considered for checking optimality conditions for optimization over the base polyhedron.

\begin{proposition}[Maximizers of support function of the base polyhedron]
\label{prop:optsupport}
Let $w \in \rb^p$.
The base $s \in B(F)$ is a  maximizer  of $\max_{s \in B(F)} s^\top w$ if and only for all $k \in V $
and $q \in { \rm Dep}(s,k)$, $w_k \leqslant w_{q}$ (i.e., for all exchangeable pairs).
\end{proposition}
\begin{proof}
If $s$ is optimal, then if $k \in V $
and $q \in { \rm Dep}(s,k)$, then for $\alpha>0$ small enough, $s'=s+ \alpha( \delta_k - \delta_{q})$ is in $B(F)$ (indeed, if $A$ is not tight, then a small modification of $s$ does not change the constraint, and if $A$ is tight, if $A \ni k$, then $q \in A$ by Prop.~\ref{prop:pairs} and thus $s'(A)=F(A)$; finally,  if $A$ tight and $k \notin A$, then $s'(A) $ can only decrease). Optimality of $s$ implies that $w_k \leqslant w_{q}$.

If the condition is true, we can order values of $w$, as $w_{B_1}>\cdots > w_{B_m}$
(where $w_k =w_{B_j}$ for $k \in B_j$). Let $A_j = B_1 \cup \cdots \cup B_j$, so that $k \in A_j$ if and only if $w_k \geqslant w_{B_j}$. This implies, because of the condition, that $A_j = \bigcup_{k \in A_j} { \rm Dep}(s,k)$, and thus that $A_j$ is tight (as a union of tight sets), i.e., $s(A_j) = F(A_j)$.
Then, for any $t \in B(F)$,
\BEAS
s^\top w - t^\top w & = & 
\sum_{k \in V} w_k (s_k - t_k) = \sum_{i=1}^m w_{B_i} [ s(B_i) - t(B_i) ] \\
 & = & 
 \sum_{i=1}^m w_{B_i} [ (s-t)(A_i) - (s-t)(A_{i-1}) ] \\ 
 & = & 
 \sum_{i=1}^m w_{B_i} [ (F-t)(A_i) - (F-t)(A_{i-1}) ] \\ 
 & = & 
 \sum_{i=1}^m [ F(A_i) - t(A_i) ] ( w_{B_i} - w_{B_{i+1}} ) \geqslant 0.
\EEAS
Thus $s$ is optimal. Note that this also a consequence of Prop.~\ref{prop:optsupporttight}.
\end{proof}

From Prop.~\ref{prop:optsupport}, we may now deduce the tangent cone of the base polyhedron, from which we then obtain optimality conditions.

\begin{proposition}[Tangent cone of base polyhedron]
\label{prop:tangent}
Let $s \in B(F)$, the tangent cone of $B(F)$ at $s$ is generated by vectors $\delta_k - \delta_{q}$ for all $k \in V $
and $q \in { \rm Dep}(s,k)$, i.e., for all exchangeable pairs $(k,q)$.
\end{proposition}
\begin{proof} Given the proof of Prop.~\ref{prop:optsupport}, each of the vectors  $\delta_k - \delta_{q}$ belongs to the tangent cone. If the tangent cone strictly contains the conic hull of these vectors, by Farkas lemma (see, e.g.,~\cite{borwein2006caa}), there exists $y$ in the tangent cone and $w \in \rb^p$, such that for all exchangeable pairs $(k,q)$, $w^\top( \delta_k - \delta_{q} )  \leqslant 0$ and $w^\top y > 0$. By the last proposition, $s$ is an optimal base for the weight vector $w$, however, $s+\alpha y \in P(F)$
 for $\alpha>0$ sufficiently small and 
 $(s+\alpha y)^\top w > s^\top w$, which is a contradiction.
 \end{proof}

\begin{proposition}[Optimality conditions for separable optimization]
\label{prop:optsep}
 Let $g_j$ be convex functions on $\rb$, $j=1,\dots,p$. Then $s \in B(F)$ is a minimizer of $\sum_{j\in V} g_j(s_j)$ over $s \in B(F)$ if and only if for all exchangeable pairs $(k,g)$, $\partial_+ g_k(s_k) \geqslant \partial_- g_{q}(s_{q})$, where $\partial_+ g_k(s_k)$ is the right-derivative of $g_k$ at $s_k$ and
 $\partial_- g_{q}(s_{q})$ is the left-derivative of $g_{q}$ at $s_{q}$.
  \end{proposition}
\begin{proof} This is immediate from Prop.~\ref{prop:tangent} related to the tangent cone of $B(F)$.
 \end{proof}

We can give an alternative description of optimality conditions based on Prop.~\ref{prop:optsupporttight}, which we give only for differentiable functions for simplicity.

\begin{proposition}[Alternative optimality conditions for separable optimization]
\label{prop:optsepalt}
 Let $g_j$ be \emph{differentiable} convex functions on $\rb$, $j=1,\dots,p$. Let $s \in B(F)$  and $w \in \rb^p$ defined as $\forall k \in V, w_k = g_k'(s_k)$; define $B(\alpha) = \{ w \leqslant \alpha\}$ for $\alpha \in \rb$. Then, $s$ is a minimizer of $\sum_{j\in V} g_j(s_j)$ over $s \in B(F)$ if and only if for all $\alpha \in \rb$, the sets $B(\alpha)$ are tight. 
   \end{proposition}
\begin{proof} Note that the condition has to be checked only for $\alpha$ belonging to the of values taken by $w$.
We consider the unique values $v_1 < \cdots < v_m $, taken at sets $A_1,\dots,A_m$ (i.e., $V = A_1 \cup \cdots \cup A_m$ and $\forall k \in A_i, \ w_k = v_i$). The condition then becomes that all $B_i = A_1\cup \dots \cup A_i$ are tight for $s$.
This  is immediate from    Prop.~\ref{prop:optsupporttight}. Indeed, $s$ is optimal if and only if $s$ is optimal for the problem $\min_{s \in B(F)} w^\top s$.
 \end{proof}

\subsection{Lexicographically optimal bases}
We can give another interpretation to optimality conditions in Prop.~\ref{prop:optsep}. Given a vector $s \in \rb^p$, we denote by $T(s) \in \rb^p$, the sequence of components of $s$ in order of increasing magnitude. That is, if $s_{j_1} \leqslant s_{j_2} \leqslant \cdots \leqslant s_{j_p}$, then $T(s)=(s_{j_1},\dots,s_{j_p})$. Given two vectors $s$ and $s'$ in $\rb^p$, $s$ is said lexicographically greater than or equal to $s'$, if either (a) $s=s'$, or, (b) $s\neq s'$, and for the minimum index $i$ such that $s_i \neq s_i'$, then $s_i \geqslant s_i'$.

We now show that finding a base $s \in B(F)$ that lexicographically maximizes the ordered vector of derivatives $g_k'(s_k)$ is equivalent to minimizing $\sum_{k \in V} g_k(s_k)$ over the base polyhedron. Many algorithms for proximal problems are in fact often cast as maximization for such lexicographical orders~(see, e.g.~\cite{megiddo1974optimal}).

\begin{proposition}[Lexicographically optimal base]
 Let $g_j$ be differentiable strictly convex functions on $\rb$, $j=1,\dots,p$. Then $s \in B(F)$ lexicographically maximizes the vector $T(g'(s)) = T[ (g_1'(s_1),\dots,g_p'(s_p)) ]$ over $s \in B(F)$ if and only if $s$ is a minimizer of $\sum_{k \in V} g_k(s_k)$ over the base polyhedron $B(F)$.
 \end{proposition}

\begin{proof}
First assume that  $s \in B(F)$ lexicographically maximizes the vector $T(g'(s)) = T[ (g_1'(s_1),\dots,g_p'(s_p)) ]$ over $s \in B(F)$. Then, for any exchangeable pair $(k,q)$ associated with $s$, we have that $t =s + \alpha (\delta_k - \delta_{q}) \in B(F)$ for $\alpha$ sufficiently small (from Prop.~\ref{prop:tangent}). Moreover,  all components $g'_j(s_j)$ are unchanged, except the $k$-th and $q$-th position, fo which we have $g'_k(t_k) >g'_k(s_k)$ and $g'_{q}(t_{q}) < g'_{q}(s_{q})$.
Thus, if $g_k'(s_k) <  g'_{q}(s_{q})$, $T(g'(t))$ is lexicographically strictly greater than $T(g'(s))$, which is a contradiction. This implies that for all exchangeable pairs, $g_k'(s_k) \geqslant  g'_{q}(s_{q})$, which implies, by Prop.~\ref{prop:optsep} that $s$ is indeed a minimizer.

Let now $s$ be a minimizer of $\sum_{k \in V} g_k(s_k)$ over the base polyhedron $B(F)$. Let $t$ be a base in $B(F)$ such that $ T(g'(t))$   is lexicographically greater than or equal to $ T(g'(s))$. We consider $v=g'(t) \in \rb^p$ and $w=g'(s) \in \rb^p$. We denote by $w_{B_1} < \cdots < w_{B_m}$ the $m$ distinct values of $w \in \rb^p$, taken on the subsets $A_j$, $j=1,\dots,m$. From Prop.~\ref{prop:optsupporttight}, the sets $B_j = A_1 \cup \cdots \cup A_j$ are tight for $s$. We show by induction on $j$ that for  $k \in B_j$, $s_k = t_k$, which will show that we must have $s=t$, and thus that $T(g'(s))$ is lexicographically optimal.

This is true for $j=0$, and if we assume it is true for $j$, then, since $T(v)$ is lexicographically greater than or equal to $T(w)$, we have for all $k \in A_{j+1}$, $v_k \geqslant w_k$ (since all the smaller ones are equal by the induction assumption), which implies, by strict convexity of $g_k$ that $t_k \geqslant s_k$. Moreover, since $B_{j+1}$ is tight, we
have
$F(B_{j+1}) \geqslant t(B_{j+1}) \geqslant s(B_{j+1}) = F(B_{j+1})$, which implies that $t_k=s_k$ for $k \in A_{j+1}$.

\end{proof}

\subsection{Optimization for proximal problems}
We can now obtain from the base polyhedron perspective the previous results linking problems in \eq{prox} and \eq{proxalpha}, i.e., give an alternative proof of Prop.~\ref{prop:subfromprox} from \mysec{prox}.

Indeed, from Prop.~\ref{prop:optsep}, $s$ is optimal for $\max_{s \in B(F)}   - \sum_{j=1}^p \psi_j^\ast(-s_j)$ if and only if for all exchangeable pairs $(k,q)$ for $s$, $(\psi_k^\ast)'(-s_k) \leqslant 
(\psi_{q}^\ast)'(-s_{q})$. If we denote $w_k = (\psi_k^\ast)'(-s_k)$ (which is equivalent to $s_k = - \psi_k'(w_k)$), then $s$ is optimal if $w_k \leqslant w_{q}$ for all exchangeable pairs $(k,q)$.

Let $\alpha \in \rb$, we consider the optimization problem
\BEQ
\label{eq:proxalphadual}
\max_{s \in B(F) } \sum_{k \in V} ( s_k + \psi_k'(\alpha) )_- =
 ( s  + \psi'(\alpha) )_-(V).
\EEQ

From Prop.~\ref{prop:optsep} and the fact that the right-derivative of $s_k \mapsto  ( s_k + \psi_k'(\alpha) )_-$ is $-1$ for $s_k < \psi_k'(\alpha)$ and zero otherwise, and its left-derivative of $s_k \mapsto  ( s_k + \psi_k'(\alpha) )_-$ is $-1$ for $s_k \leqslant - \psi_k'(\alpha)$ and zero otherwise, 
 $s$ is optimal if and only if for all exchangeable pairs $(k,q)$ for $s$, we have
 $1_{ \{ s_k < - \psi_k'(\alpha) \}} \leqslant  1_{ \{ s_{q} \leqslant - \psi'_{q}(\alpha) \}}$, which is equivalent to the fact that $s_k<- \psi_{k}'(\alpha)$ implies that $s_{q} \leqslant - \psi_{q}'(\alpha)$.
 
 If $s$ is optimal for \eq{proxdual}, then $(\psi_k^\ast)'(-s_k) \leqslant 
(\psi_{q}^\ast)'(-s_{q})$ for all exchangeable pairs.
  Thus, if $s$ is optimal for \eq{proxdual}, then $s$ is optimal for the maximization of \eq{proxalphadual} for all $\alpha \in \rb$.

Finally, from Prop.~\ref{prop:dualmin}, solving \eq{proxalphadual} is equivalent to minimizing the submodular function $F+ \psi'(\alpha)$, which is exactly \eq{proxalpha}. Also, from Prop.~\ref{prop:dualmin}, we have that any optimal $A^\alpha$ satisfies $ \{ s + \psi'(\alpha) <0 \} \subset A^\alpha \subset \{ s + \psi'(\alpha) \leqslant 0 \} $. Moreover, since at the optimum, $w_k + \psi_k'(s_k)=0$, we thus have $ s_k + \psi'_k(\alpha) <0$ if and only if
$w_k > \alpha$, and $ s_k + \psi'_k(\alpha) \leqslant 0$ if and only if
$w_k \geqslant \alpha$. We thus get back Prop.~\ref{prop:subfromprox}.

\section{Submodular function minimization}
\label{sec:sfm}

Several generic algorithms may be used for the minimization of a submodular function. They are all based on a sequence of evaluations of $F(A)$ for certain subsets $ A \subset V$. For specific functions, such as the ones defined from cuts, faster algorithms exist (see, e.g., \cite{gallo1989fast,hochbaum2001efficient} and \mysec{cuts}).

Note that maximizing submodular functions is a hard combinatorial problem in general. However, when maximizing a non-decreasing submodular function under a cardinality constraint, the simple greedy method allows to obtain a $(1-1/e)$-approximation~\cite{nemhauser1978analysis}.

In this section, we first review classical approaches for submodular function minimization. The first approach presented in \mysec{minnorm} is the most efficient in practice, but has no complexity bound. We briefly mention in \mysec{comb} existing combinatorial algorithms with theoretical complexity bounds, but these are not used in practice. In \mysec{posi}, we consider certain submodular functions, so-called posimodular functions, for which simple combinatorial algorithms exist with better complexity.

We then present algorithms which are based on a sequence of submodular function minimization, and that can be used for problems such as line search in the submodular polyhedron  or proximal problems.

\subsection{Minimum-norm point algorithm}
\label{sec:minnorm}
From \eq{proxalpha} or Prop.~\ref{prop:subfromprox}, we obtain that if we know how to minimize $f(w) + \frac{1}{2} \| w\|_2^2$, or equivalently, minimize $\frac{1}{2} \| s\|_2^2$ such that $s \in B(F)$, then we get all minimizers of $F$ from the negative components of $s$.

The minimum-norm point algorithm computes the minimum of $\| s\|_2^2$ for $s \in B(F)$. It uses an old algorithm from~\cite{wolfe1976finding} that will find a minimum-norm base $s \in B(F)$ in a finite number of steps.  This is made possible by the fact that we know how to efficiently maximize linear functions over $B(F)$, where solutions are obtained by the greedy algorithm from Prop.~\ref{prop:greedy}. 

The complexity of each step of the algorithm is essentially $O(p)$ function evaluations and operations of order $O(p^3)$. However, there are no known upper bounds on the number of iterations.

Note that
once we know which values of the optimum values $s$ should be equal, greater or smaller, then, we obtain in closed form all values.
Indeed, let $c_1 < c_2 < \cdots < c_m$ the $m$ different values taken by $s$ (or $w$), and $A_i$ the corresponding sets such that $w_{k} = c_j$ for $k \in A_j$. We then have:
$$
c_j = \frac{ f( A_1 \cup \cdots \cup A_{j}) -  f( A_1 \cup \cdots \cup A_{j-1}) }{ | A_{j}| }
$$
which allows to compute the values $c_j$ knowing only the sets $A_j$.

\subsection{Combinatorial algorithms}
\label{sec:comb}
Algorithms are based on Prop.~\ref{prop:dualmin}, i.e., on the identity $\min_{A\subset V} F(A) = \max_{ s \in B(F) } s_-(V)$.  Combinatorial algorithms will usually output the subset $A$ and a base $s \in B(F)$ such that $A$ is tight for $s$ and $\{ s < 0 \} \subset A \subset \{ s \leqslant 0\}$, as a certificate of optimality.

Most algorithms, will also output the largest   minimizer $A$ of $F$, or sometimes describe the entire lattice of minimizers.
Best algorithms have polynomial complexity~\cite{schrijver2000combinatorial,iwata2001combinatorial,orlin2009faster}, but still have high complexity (typically $O(p^6)$ or more).

\subsection{Minimizing posimodular functions}
\label{sec:posi}
A submodular function $F$ is said symmetric if for all $B \subset V$, $F(V \backslash B)  = F(B)$. By applying submodularity, get that $2 F(B) = F(V \backslash B)  +  F(B) \geqslant F(V) + F( \varnothing) = 2 F(\varnothing) = 0$, which implies that $F$ is non-negative. Hence its global minimum is attained at $V$ and $\varnothing$.

Such functions can be minimized in time $O(p^3)$ over all \emph{non-trivial} (i.e., different from $\varnothing$ and $V$) subsets of $V$~\cite{queyranne1998minimizing}. Moreover, the algorithm is valid for the regular minimization of \emph{posimodular} functions~\cite{nagamochi1998note}, i.e., of functions that satisfies
$$
\forall A,B \subset V, \ F(A) + F(B) \geqslant F( A \backslash B) +  F( B \backslash A).
$$
These include symmetric submodular functions as well as modular functions, and hence the sum of any of those (in particular, cuts with sinks and sources, as presented  in \mysec{cuts}).

\subsection{Line search in submodular polyhedron}
\label{sec:line}

The general line search problem in the submodular polyhedron amounts to start from $s \in P(F)$ and search on the direction $t \in \rb^p$, i.e., find the maximal $\lambda \geqslant 0$ such that $s+ \lambda t \in P(F)$, which is equivalent to $\lambda t \in P(F - s)$. Note that since $s \in P(F)$, $F-s$ is submodular and non-negative.

We thus now assume that $F$ is non-negative and that $s=0$.
Given $t \in \rb^p$, we consider the problem of finding the largest $\lambda \geqslant 0$ such that 
$\lambda t \in P(F)$. We denote by $\mu$ the optimal value (which is finite, as soon as there is at least one $t_k >0$, which we assume). We have
$\lambda \leqslant \mu$ if and only if $g(\lambda) = \min_{A \subset V} F(A) - \lambda t(A) \geqslant 0$. More precisely,  $g(\lambda) \geqslant 0$ if and only if $\lambda t\in P(F)$. Moreover, $g(0)=0$ and $g$ is non-increasing, which implies that $g$ is zero on $[0,\mu]$ and then strictly negative.

We thus need to find the zero of the function $g(\lambda)$, which is piecewise affine. This can be done with the secant method, once we have a $\lambda>0$ such that $g(\lambda)< 0$. Such a $\lambda$ can be obtained by noting that $P(F)$ is included in $\{ s, \forall k \in V, s_k \leqslant F(\{k\})\}$, which implies that if $\lambda > \min_{k \in V} \frac{F(\{k\})}{t_k}$, then $g(\lambda)<0$. 

The secant method is simply starting with a $\lambda$ such that $g(\lambda)>0$, and then find the minimizer $A$ in the definition of $g(\lambda)$, and set $\lambda = F(A) / t(A)$, and start again in $g(\lambda)<0$ (see~\cite{nagano2007strongly} for more details). Note that if the minimum-norm point algorithm is used for submodular function minimization, then we obtain instead a minimizer
of $w \mapsto f(w)  - \lambda w^\top t + \frac{1}{2} \| w\|_2^2$, and we can also update $\lambda$ as
$\lambda = ( f(w) + \frac{1}{2} \| w\|_2^2 )  / ( w^\top t)$.

\subsection{Homotopy method  for proximal problems}
\label{sec:proxcomb}
We review in \mysec{proxcomb} and \mysec{decomp} two strategies for maximizing separable concave functions on the base polyhedron. One strategy is based on the equivalence with the sequence of minimizations of submodular functions (Prop.~\ref{prop:proxmin}). The other one is based on a decomposition strategy.

 The first method is based on the fact that if $\alpha$ is large enough, then $A^\alpha = \varnothing$ is optimum for \eq{proxalpha}. From Prop.~{prop:dualmin}, this is valid as long as $0 \in P(F +\psi'(\alpha))$, i.e., $-\psi'(\alpha) \in P(F)$. The minimum $\alpha \in \rb$ such that this is valid can be obtained by line search.
 
 Once the minimal $\alpha$ is found, and $A$ is the maximal tight set associated with $-\psi'(\alpha)$, then if $A = V$, $w = \alpha 1_V$. Otherwise, we let $w_A = \alpha 1_A$, and in order to determine $w_{V \backslash A}$ we recursively apply the same procedure to the function $F_{V \backslash A}: 2^{V \backslash A} \to \rb$, defined as $F_{V \backslash A}(B)= F(B)$ (i.e., restriction of $F$ to $V \backslash A$).
 
 This algorithm, adapted from~\cite{fujishige1980lexicographically} (see also~\cite[Sec.~9.2]{fujishige2005submodular}), requires to be able to find the minimum $\alpha$ such that $-\psi'(\alpha) \in P(F)$. This may be done as follows (same procedure as in \mysec{line}, but extended to non quadratic functions).

 Consider $g(\alpha) = \min_{A \subset V} F(A) + \psi'(\alpha)(A)$. The function is piecewise smooth and strictly increasing. It is equal to zero if and only if $-\psi'(\alpha) \in P(F)$, and it is strictly negative otherwise. We start with a point $\alpha_0$ such that $g(\alpha_0)<0$, we let $A_0$ be a minimizer in the definition of $g(\alpha_0)$. We find the unique $\alpha_1$ such that  $F(A) + \psi'(\alpha_1)(A_0)=0$ and we start again, until we have $g(\alpha_1)=0$.
 
 In order to find $\alpha_0$ such that $g(\alpha_0)<0$, we use the fact that $P(F) \subset \prod_{k \in V} (-\infty; F(\{k\})]$, and thus if there exists $k \in V$, $\psi'_k(\alpha)>-F(\{k\})$, then $-\psi'(\alpha) \notin P(F)$. We can thus consider $\alpha_0 = \min_{ k \in V} (\psi'_k)^{-1}(-F(\{k\}))$.

\subsection{Decomposition algorithm for proximal problems}
\label{sec:decomp}
We adapt the algorithm of~\cite{groenevelt1991two} and~\cite[Sec.~8.2]{fujishige2005submodular}. Note that it can be slightly modified for problems with non-decreasing submodular functions~\cite{groenevelt1991two} (see also \mysec{polym}).

For simplicity, we consider \emph{strictly convex differentiable} functions $g_j$, $j=1,\dots,p$, and the following algorithm:

\BNUM
\item Find the unique minimizer $t \in \rb^p$ of  $\sum_{j\in V} g_j(t_j)$ such that $t(V) = F(V)$.
\item Minimize the submodular function $F-t$, i.e., find the \emph{largest} $A \subset V$  that minimizes $F(A)-t(A)$.
%and $q \in B(F-t)$, such that
%$q_-(V) = F(A) - t(A)$, i.e.,  $u \in B(F)$ such that $(u-t)_-(V) = F(A) - t(A)$.
\item If $A = V$, then $t$ is optimal. Exit.
\item Find a minimizer $s_A$ of $\sum_{j \in A} g_j(s_j)$ over $s$ in the base polyhedron associated to $F_A$,
the restriction of $F$ to  $A$.
\item Find a minimizer $s_{V \backslash A }$ of $\sum_{j \in V \backslash A} g_j(s_j)$ over $s$ in the base polyhedron associated to  the contraction $F^A$ of $F$ on A, defined as $F^A(B) = F(A \cup B) - F(A)$.
\item Concatenate $s_A$ and  $s_{V \backslash A }$. Exit.
\ENUM

The algorithm must stop after \emph{at most} $p$ iterations. Indeed, if $A\neq V$ in Step 3, then we must have $A \neq \varnothing$ (indeed, $A= \varnothing$ implies that $t \in P(F)$, which in turns implies that $A = V$ because by construction $t(V)=F(V)$, which leads to a contradiction). Thus we actually split $V$ into two non-trivial parts $A$ and $V \backslash A$.

We now  need to prove optimality. Let $s$ be the output of the algorithm. We first show that $s \in B(F)$. We have for any $B \subset V$:
\BEAS
s(B) & = & s(B \cap A) + s( B \cap ( V \backslash A) ) \\
&  \leqslant &   F(B \cap A) + F( A \cup B) - F(A) \mbox{ by definition of } s_A \mbox{ and } s_{V \backslash A }\\
& \leqslant & F(B) \mbox{ by submodularity}.
\EEAS
Thus $s$ is indeed in the submodular polyhedron $P(F)$. Moreover, we have
$s(V) = s_A(A) + s_{V \backslash A }(V \backslash A ) = F(A) + F(V) - F(A) = F(V)$, i.e., $s$ is in the base polyhedron $B(F)$.

 We now construct a second base $\bar{s} \in B(F)$ as follows: $\bar{s}_A$ is the minimizer of $\sum_{j \in A} g_j(s_j)$ over $s$ in the base polyhedron associated to the submodular polyhedron $P(F_A) \cap \{ s_A \leqslant t_A \}$. From Prop.~\ref{prop:conv}, the associated submodular function is
 $H_A(B) = \min_{C \subset B} F(C) + t( B \backslash C)$. We have
 $H_A(A) = \min_{C \subset A} F(C) - t(C)  + t( A ) = F(A)$ because $A$ is the largest minimizer of $F-t$. Thus, the base polyhedron associated with $H_A$ is simply $B(F_A) \cap \{ s_A \leqslant t_A \}$. Moreover, from Prop.~\ref{prop:conv}, we have that $H_A \leqslant F_A$, and thus if $s_A$ is tight for $F_A$ then $s_A$ is tight for $H_A$.

 Morover, we define $\bar{s}_{V \backslash A }$ as the minimizer of $\sum_{j \in V \backslash A} g_j(s_j)$ over the base polyhedron $B(J^A)$ where we define the submodular function $J^A$ on $V \backslash A$ as follows:
 $J^A(B)  = \min_{ C \supset B } F(C \cup A) - F(A) -t(C) + t(B)$. Then $J^A - t$ is non-decreasing and submodular (by Proposition~\ref{prop:monotone}). Moreover, $J^A(V \backslash A) = F(V) - F(A)$ and $J^A \leqslant F^A$. Finally
 $B(F^A) \cap \{ s_{V \backslash A }\geqslant t_{V \backslash A } \} = B(J^A)$ and thus if $s_A$ is tight for $F^A$ then $s_A$ is tight for $J^A$.

 We now show that $\bar{s}$ is optimal for the problem. Since $\bar{s}$ has a higher objective value than~$s$, the base $s$ will then be optimal as well. If we take an exchangeable pair $(k,q)$ for $\bar{s}$. Then, we have several cases (note that $A$ is tight for $\bar{s}$):
 \BIT
 \item $k \in A$, implies $q \in A$ (by Prop.~\ref{prop:pairs}, since $A$ is tight), and thus the optimality condition stems from the sub-problem on $A$ (since being tight for $F_A$ implies being tight for $H$)
 \item $  k \notin A$, $q\in A$, it comes from $\bar{s}_A \leqslant t_A$ and $\bar{s}_{V \backslash A }\geqslant t_{V \backslash A }$, which implies $g_k'(\bar{s}_k) \geqslant g_{q}'(\bar{s}_{q})$ (since all $g_k'(t_k)$ are equal by definition of $t$).
 \item $ k \notin A$, $q \notin A$, it comes from the optimality of the subproblem on $V \backslash A $, (since being tight for $F^A$ implies being tight for $J^A$).
 \EIT
 
 In all cases, for exchangeable pairs $(k,q)$, we have $g_k'(\bar{s}_k) \geqslant g'_{q}(\bar{s}_{q})$ and thus, by Prop.~\ref{prop:optsep}, $\bar{s}$ is optimal and hence $s$ is optimal. Note that we could also have used Prop~\ref{prop:optsepalt} to show optimality.
 
 Note finally that similar algorithms may be applied when we restrict $s$ to be integers (see, e.g.,~\cite{groenevelt1991two, hochbaum2001efficient}).

\section{Polymatroids (non-increasing submodular functions)}
\label{sec:polym}

When the submodular function $F$ is also \emph{non-decreasing}, i.e., when for $A, B \subset V$, $A \subset B \Rightarrow F(A) \leqslant F(B)$, then a truncated greedy algorithm may be applied for all linear functions (i.e., with potentially negative coefficients). Such non-decreasing and submodular functions are often referred to as  \emph{polymatroid set-functions}~\cite{fujishige2005submodular} or \emph{$\beta$-functions}~\cite{edmonds}. Note that in this situation, the \lova extension is non-decreasing with respect to all components, i.e., if $w \leqslant w'$, then $f(w) \leqslant f(w')$.

\begin{proposition}[Truncated greedy algorithm]
\label{prop:truncgreedy}
Assume $F$ is submodular and non-decreasing. Let $w \in \rb^p$; a maximizer of $\max_{ s \in P(F), \ s \geqslant 0 } w^\top s$ may be obtained by the following algorithm: order all the strictly positive components of $w$, as $w_{j_1} \geqslant \cdots \geqslant w_{j_m} >0 $ and define
$s_{j_k} =  F(\{j_1,\dots,j_k\}) - F(\{j_1,\dots,j_{k-1}\})$ for $k\leqslant m$, and zero otherwise. Moreover, 
 $\max_{ s \in P(F), \ s \geqslant 0 } w^\top s = f(w_+)$.
\end{proposition}
\begin{proof}
The proof is similar to that of Prop.~\ref{prop:greedy}. The constraint $w_k = \sum_{A \ni k} \lambda_A$ is simply replaced
by $w_k  \leqslant \sum_{A \ni k} \lambda_A$ (because of the new constraint $s \geqslant 0$). The vector $s$ is then feasible because of the monotonicity of $F$.
\end{proof}

We can also specialize several other results to polymatroids. In this setting, it is easy to see that the base polyhedron $B(F)$ is included in positive orthant $\rb_+^p$ (this is for example a consequence of the greedy algorithm from Prop.~\ref{prop:greedy}). However, $P(F)$ is not included in the positive orthant, and it is common to consider the positive polyhedron
$$
P_+(F) = P(F) \cap \rb_+^p = \{ s \geqslant 0, \ \forall A \subset V, s(A) \leqslant F(A) \},
$$
which is compact (while $P(F)$ is never, as it is unbounded).

We now extend Prop.~\ref{prop:optsupporttight} and Prop.~\ref{prop:optsupporttightSUB} related to support functions, to the independence polyhedron $P_+(F)$, as well as proposition Prop.~\ref{prop:faces}, related to faces of the polyhedron.

\begin{proposition}[Maximizers of the support function of independence polyhedron]
\label{prop:optsupporttight_indep}
Let $F$ be a non-decreasing submodular function such that $F(\varnothing)=0$. Let $w \in \rb^p$, with unique values $v_1 > \cdots > v_m$, taken at sets $A_1,\dots,A_m$. Then $s$ is optimal for $\max_{s \in P_+(F)} w^\top s$ if and only if for all $i=1,\dots,m$, $v_i < 0 \Rightarrow s_{A_i}=0$, and $v_i \geqslant 0 \Rightarrow s(A_1 \cup \cdots \cup A_i) = F(A_1 \cup \cdots \cup A_i)$.
 \end{proposition}
\begin{proof} The proof follows the same arguments than for Prop.~\ref{prop:optsupporttightSUB}, with a special treatment for the negative values of $w$.
 \end{proof}

\begin{proposition}[Faces of the independence polyhedron]
\label{prop:faces_indep}
Let $F$ be a non-decreasing submodular function such that $F(\varnothing)=0$. Let $B$ be a stable set (i.e., such that all strict larger subsets have  strictly greater function values), and $A_1 \cup \cdots \cup A_m$ an ordered partition of $B$, such that for all $j \in \{1,\dots,m\}$, $A_j$ is inseparable for the function $G_j: B \mapsto F( A_1 \cup \cdots  \cup A_{j-1} \cup B) - F( A_1 \cup \cdots  \cup A_{j-1})$ defined on subsets of $A_j$, then  the set of $s \in P_+(F)$ such that for all $j \in \{1,\dots,m\}$, $s(A_1 \cup \cdots \cup A_i) = F( A_1 \cup \cdots  \cup A_i)$, and $s_{V \backslash B}=0$, is a proper face of $P_+(F)$ with non-empty relative interior.
\end{proposition}

\begin{proof}
We have a face from Prop.~\ref{prop:optsupporttight_indep}, and it has non empty interior by applying
Prop.~\ref{prop:fulldim} on each submodular function $G_j$, and using the stability of $B$.
\end{proof}

We now show how to minimize a separable convex function on the submodular polyhedron or the positive submodular polyhedron (rather than on the base polyhedron).
We first show the following proposition for the submodular polyhedron of any submodular function (non necessarily non-decreasing).

\begin{proposition}[Separable optimization on the submodular polyhedron]
\label{prop:sepsub}
Assume that $F$ is submodular.
Let $\psi_j$, $j=1,\dots,p$ be $p$ convex functions such that $\psi_j^\ast$ is defined and finite on $\rb$. Let $(v,t)$ be a primal-dual optimal pair for the problem
$$\min_{v \in \rb^p } \max_{ t \in B(F)} t^\top v + \sum_{k \in V} \psi_k(v_k)
= \min_{v \in \rb^p} f(v) + \sum_{k \in V} \psi_k(v_k)
=  \max_{ t \in B(F)} - \sum_{k \in V} \psi_k^\ast(-t_k). $$

For $k \in V$, let $s_k$ be a maximizer of $-\psi_k^\ast(-s_k)$ on $(-\infty,t_k]$. Define $w = v_+$.  Then $(w,s)$ is a primal-dual optimal pair for
the problem
$$\min_{w\in \rb^p } \max_{ s \in P(F)} s^\top w + \sum_{k \in V} \psi_k(w_k)
= \min_{w \in \rb^p_+} f(w) + \sum_{k \in V} \psi_k(w_k)
=  \max_{ s \in P(F)} - \sum_{k \in V} \psi_k^\ast(-s_k). $$
\end{proposition}
\begin{proof}
 The pair $(w,s)$ is optimal if and only if $ w_k s_k + \psi_k(w_k) + \psi_k^\ast(-s_k) = 0$, i.e., $(w_k,s_k)$ is a Fenchel-dual pair for $\psi_k$, and $f(w) = s^\top w$. The first statement is true by construction (indeed, if $s_k = t_k$, then this is a consequence of optimality for the first problem, if $s_k<t_k$, then $w_k = (\psi_k^\ast)'(-s_k) = 0$).
 
 For the second statement,   notice that $s$ is obtained from $t$ by keeping the components of $t$ corresponding to strictly positive values of $v$ (let $K$ denote that subset), and lowering the ones for $V \backslash K$. For $\alpha > 0$, the level sets $\{ w \geqslant \alpha \}$ are equal to $\{ v \geqslant \alpha\} \subset K$. Thus, by Prop.~\ref{prop:optsupporttight},  all of these are tight for $t$ and hence for $s$ because these sets are included in $K$, and $s_K = t_K$. This shows, by  Prop.~\ref{prop:optsupporttightSUB}, that $s \in P(F)$ is optimal for $\max_{ s \in P(F)} w^\top s$.
  \end{proof}

Note that Prop.~\ref{prop:sepsub} involves primal-dual pairs $(w,s)$ and $(v,t)$, but that we can define $w$ from $v$ only, and define $s$ from $t$ only; thus,  primal-only views and dual-only views are possible. This also applies to Prop.~\ref{prop:sepsubpos}.

 \begin{proposition}[Separable optimization on the positive submodular polyhedron]
\label{prop:sepsubpos}
Assume that $F$ is submodular and non-increasing.
Let $\psi_j$, $j=1,\dots,p$ be $p$ convex functions such that $\psi_j^\ast$ is defined and finite on $\rb$. Let $(v,t)$ be a primal-dual optimal pair for the problem
$$\min_{v \in \rb^p } \max_{ t \in B(F)} t^\top v + \sum_{k \in V} \psi_k(v_k)
= \min_{v \in \rb^p} f(v) + \sum_{k \in V} \psi_k(v_k)
=  \max_{ t \in B(F)} - \sum_{k \in V} \psi_k^\ast(-t_k). $$

For $k \in V$, let $s_k$ be a maximizer of $-\psi_k^\ast(-s_k)$ on $[0,t_k]$. For all $k$,  define $w_k$ through
$s_k + \psi_k'(w_k)=0$.  Then $(w,s)$ is a primal-dual optimal pair for
the problem
$$\min_{w\in \rb^p } \max_{ s \in P_+(F)} s^\top w + \sum_{k \in V} \psi_k(w_k)
= \min_{w \in \rb^p} f(w_+) + \sum_{k \in V} \psi_k(w_k)
=  \max_{ s \in P_+(F)} - \sum_{k \in V} \psi_k^\ast(-s_k). $$
\end{proposition}
\begin{proof} We first apply Prop~\ref{prop:sepsub} to the convex functions $\tilde{\psi}_k(w_k) = \min_{ v_k \leqslant w_k} \psi_k(v_k)$, which Fenchel-conjugates equal to $\psi_k^\ast(s_k)$ if $s_k \leqslant 0$ and $+ \infty$ otherwise.
We obtain the minimum over $\rb^p_+$ of $f(w) + \sum_{j \in V} \tilde{\psi}_k(w_k)$. Since $f$ non-decreasing with respect to each variable taken separately (because $F$ is non-decreasing), it is equivalent to minimizing on $\rb^p$, $\min_{w \in \rb^p} f(w_+) + \sum_{k \in V} \psi_k(w_k)$.
 \end{proof}

\section{Examples of submodular functions}
\label{sec:examples}
We now present classical examples of submodular functions. For each of these, we also describe the corresponding \lova extensions, and, when appropriate, the associated submodular polyhedra.

\subsection{Cardinality-based functions}
We consider functions that depend only on $s(A)$ for a certain $s \in \rb_+^p$. If $s = 1_V$, these are functions of the cardinality. The next proposition shows that only concave functions lead to submodular functions, and is coherent with the diminishing return property from \mysec{definitions} (Prop.~\ref{prop:firstorder}).

\begin{proposition}[Submodularity of cardinality-based set-functions]
If $s \in \rb^p_+$ and $g:\rb_+ \to \rb$ is a concave function, then $F:A \mapsto g( s(A) )$  is submodular. If $F:A \mapsto g( s(A) )$  is submodular for all $s \in \rb_+^p$, then $g$ is concave.
\end{proposition}
\begin{proof} The function
$F:A \mapsto g( s(A) )$ is submodular if and only if for all $A \subset V$ and $j,k \in V \backslash A $: 
$g(s(A) + s_k ) - g(s(A)) \geqslant g(s(A) + s_k +s _j) -  g(s(A) + s_j)$. If $g$ is concave and $a \geqslant 0$, $ t \mapsto g(a+t)-g(t) $ is non-increasing, hence the first result. Moreover, if $ t \mapsto g(a+t)-g(t) $ is  non-increasing for all $a \geqslant 0$, then $g$ is concave, hence the second result.
\end{proof}

\begin{proposition}[\lova extension of cardinality-based set-functions]
Let $s \in \rb^p_+$ and $g:\rb_+ \to \rb$ be a concave function such that $g(0)=0$, the \lova extension of the submodular function $F:A \mapsto g( s(A) )$  is equal to
$$
f(w) =  \sum_{k=1}^p w_{j_k} [ g( s_{j_1}+ \cdots + s_{j_k}) -  g( s_{j_1}+ \cdots + s_{j_{k-1}}) ].
$$
If $s = 1_V$, i.e., $F(A) = g(|A|)$, then 
$f(w) =  \sum_{k=1}^p w_{j_k} [ g(k) -g(k-1) ]$.
\end{proposition}
 The \lova extension is thus a function of order statistics.

\subsection{Cut functions}
\label{sec:cuts}
Given a set of (non necessarily symmetric) weights $d:V \times V \to \rb_+$, define  
$$F(A) = \sum_{k \in A, j \in V \backslash A} d(k,j),$$
 which we denote $d(A,V \backslash A)$. 
Note that for a cut function and disjoint subsets $A,B,C$, we always have:
\BEAS
F(A \cup B \cup C) & = &  F(A \cup B )  + F(A   \cup C)  + F( B \cup C)  - F(A )  - F(  B  )  - F(  C)  + F(\varnothing) \\
F( A \cup B) & = & d(A \cup B,  (A \cup B)^c ) = d(A, A^c \cap B^c) + d(B, A^c \cap B^c)
\\
& \leqslant &  d(A,A^c)+ d(B,B^c) = F(A) + F(B),
\EEAS
where we denote $A^c = V \backslash A$.
We then have, for any sets $A,B \subset V$:
\BEAS 
F(A \cup B) 
& = & F([A \cap B] \cup [A \backslash B ]\cup [B \backslash A]) \\
& = & F([A \cap B] \cup [A \backslash B ] ) + F([A \cap B] \cup    [B \backslash A]) + F(  [A \backslash B ]\cup [B \backslash A]) \\
& &  - 
F(A \cap B ) - F(  A \backslash B ) - F(  B \backslash A) + F(\varnothing)  \\
& = & F(A) + F(B) + F( A \Delta B )  - 
F(A \cap B ) - F(  A \backslash B ) - F(  B \backslash A)  \\
& = & F(A) + F(B)   - 
F(A \cap B ) + [ F( A \Delta B )- F(  A \backslash B ) - F(  B \backslash A) ] \\
& \leqslant & F(A) + F(B)    - 
F(A \cap B ), \EEAS
which shows submodularity.
Moreover, the \lova extension is equal to 
$$f(w) = \sum_{k,j \in V} d(k,j) ( w_k - w_j )_+.$$ Then, if  the weight function $d$ is symmetric, then the submodular function is also symmetric and the \lova extension is even (from Prop.~\ref{prop:lova}). Examples of such cuts are shown in Figure~\ref{fig:cuts} (left and middle). A instance of these \lova extensions plays a crucial role in signal and image processing; indeed, for a graph composed a two-dimensional grid with $4$-connectivity (see \myfig{2dgrid}), we obtain the total variation. In fact, some of the results presented in this tutorial were first tackled on this particular case (see, e.g.,~\cite{chambolle2009total} and references therein).

\begin{figure}

\begin{center}
 
\includegraphics[scale=.7]{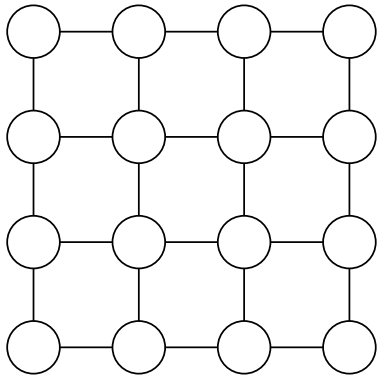} 
 \end{center}

\caption{Two-dimensional grid with $4$-conenctivity.}
\label{fig:2dgrid}
\end{figure}

We can also consider partial minimization to obtain ``regular functions''~\cite{boykov2001fast}. Examples lead to
$f(w) = \max_{k \in G}w_k - \min_{k \in G} w_k$, which corresponds
to $F(A) = 1_{ A \cap G \neq \varnothing }  - 1_{ A \cap G = \varnothing } $.

It may also lead to ``noisy cuts'', i.e., for a given a weight function $d: V \times V \to \rb_+$, we add $p$ nodes, each of them associated to the original nodes, and consider the convex and submodular functions
\BEAS
f(w) &  = &  \min_{v \in \rb^p}  \sum_{k,j \in V} d(k,j) ( v_k - v_j )_+ + \lambda \sum_{k \in V}  \alpha_k |v_k - w_k|,
\\
F(A) &  = &  \min_{B \subset V}  \sum_{k \in B,j \in B^c} d(k,j) + \lambda \sum_{k \in V} \alpha_k |1_{k \in A} - 1_{k \in B}|, 
\EEAS
 which are associated to each other due to Prop.~\ref{prop:partial}. An example of such cut is shown in Figure~\ref{fig:cuts} (right).

This example is particularly interesting, because it leads to a family of   submodular functions for which dedicated fast algorithms exist. Indeed, minimizing the cut functions or the partially minimized cut, plus a modular function defined by $z \in \rb^p$, may be done with a min-cut/max-flow algorithm (see, e.g.,~\cite{cormen89introduction}). Indeed, following~\cite{boykov2001fast,chambolle2009total}, we add two nodes to the graph, a source $s$ and a sink $t$. All original edges have non-negative capacities $d(k,j)$, while, the edge that links the source $s$ to the node $k \in V$ has capacity $(z_k)_+$ and the edge that links the node $k \in V$ to the sink $t$ has weight $-(z_k)_-$ (see bottom line of \myfig{cuts}). Finding a minimum cut or maximum flow in this graph leads to a minimizer of $F-z$.
 
 For proximal methods, such as defined in \eq{proxalpha} (\mysec{prox}), we have $z = \psi(\alpha)$ and we need to solve an instance of a \emph{parametric max-flow} problem, which may be done using efficient dedicated algorithms~\cite{gallo1989fast,hochbaum2001efficient,chambolle2009total}. See also \mysec{proxcomb} for generic algorithms based on a sequence of singular function minimizations.

\begin{figure}

\begin{center}

\hspace*{.7cm}
\includegraphics[scale=.6]{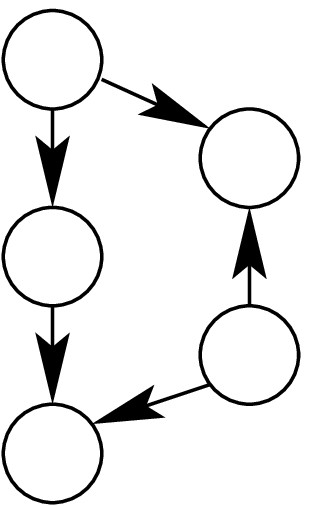} \hspace*{2.5cm}
\includegraphics[scale=.6]{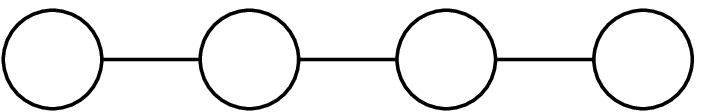}  \hspace*{1cm}
\includegraphics[scale=.6]{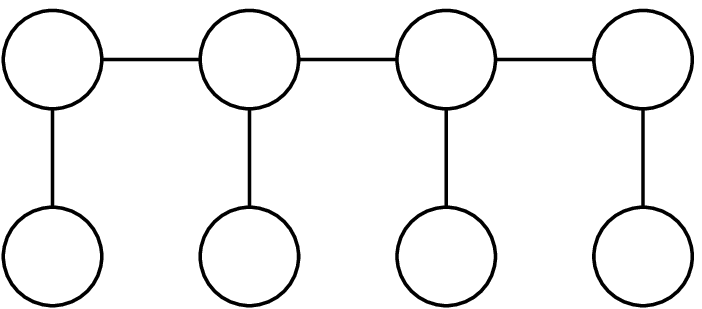}

\vspace*{.8cm}

\includegraphics[scale=.6]{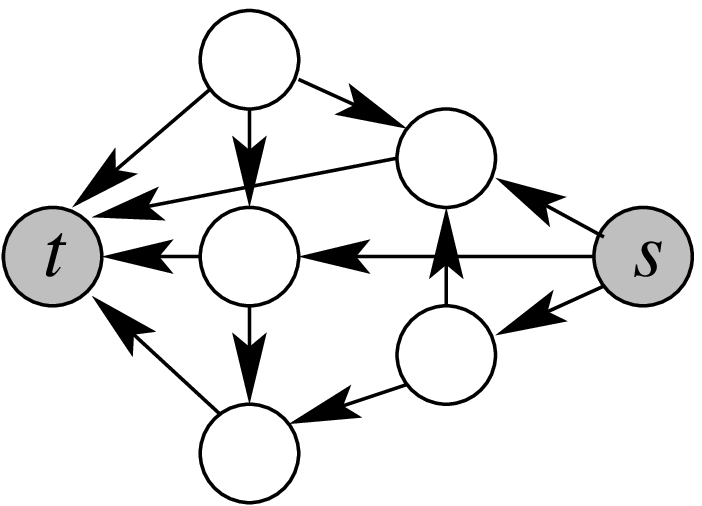} \hspace*{1cm}
\includegraphics[scale=.6]{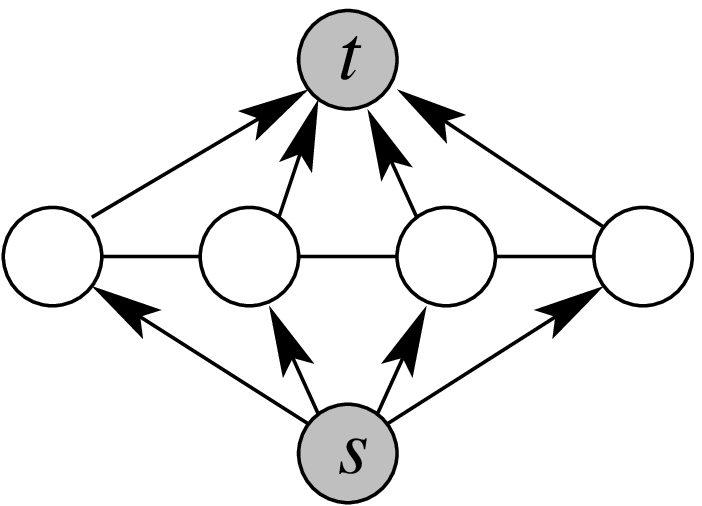}  \hspace*{1cm}
\includegraphics[scale=.6]{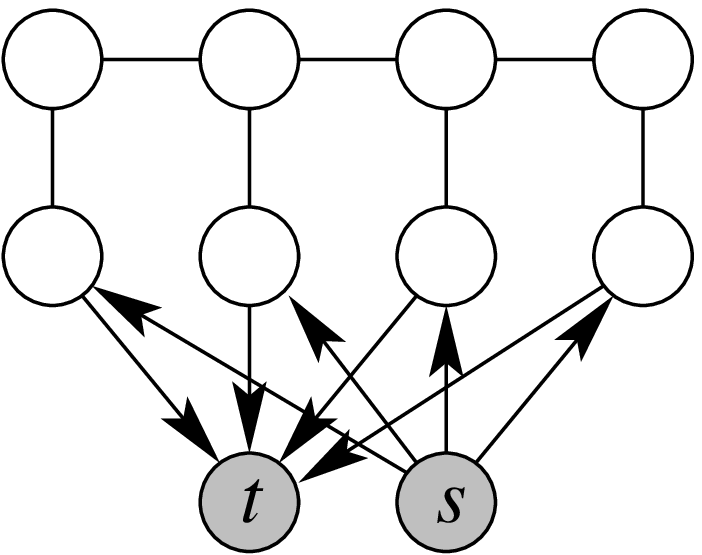}
 \end{center}

\caption{Top: graphs for symmetric (left) and non-symmetric cost functions. Bottom: corresponding networks (note that for the right plot, this corresponds to a partial minimization, we refer to in the text as noisy cuts).}
\label{fig:cuts}
\end{figure}
 
 \subsection{Set covers}
\label{sec:cover}
Given a \emph{non-negative} function $D: 2^V \to \rb_+$, then we can define 
$$F(A) = \sum_{G \subset V, G  \cap A \neq \varnothing} { \rm Dep}(G),$$
 with $f(w) = \sum_{G \subset V} { \rm Dep}(G) \max_{k \in G} w_k$. The submodularity and the \lova extension can be obtained using linearity and the fact that the \lova extension of $A \mapsto 1_{G \cap A = \varnothing}$ is $w \mapsto \max_{ k \in G } w_k$.

\paragraph{M\"obius inversion.}
Note that any set-function $F$ may be written as 
$$ F(A) = \sum_{G \subset V, G  \cap A \neq \varnothing} { \rm Dep}(G)
=\sum_{ G \subset V  } { \rm Dep}(G)   -  \sum_{ G \subset V \backslash A } { \rm Dep}(G), $$ for a certain set-function~$D$, \emph{which is not usually non-negative}. Indeed, by M\"obius inversion formula (see, e.g.,~\cite{mobius}), we have:
$${ \rm Dep}(G) = \sum_{ A \subset G} (-1)^{|G| - |A|} \big[ 
F(V) - F(A) \big] 
.$$
Thus, functions for which $D$ is non-negative are a specific subset of submodular functions. Moreover, these functions are always non-decreasing. Such functions are used in the context of sparsity-inducing norms~\cite{bach2010structured, jenattonmairal,mairal10}.

\paragraph{Reinterpretation in terms of set-covers.}
Let $W$ be any ``base'' set.
Given for each $k \in V$, a set $S_k \subset W$, we define $F(A) = \big| \bigcup_{k \in A} S_k \big|$. More generally, we can define
$F(A) = \sum_{j \in W} \Delta(j) 1_{ \exists k \in A, S_{k} \ni j}$, if we have weights $\Delta(j) \in \rb_+$ for $j \in W$ (this corresponds to replace the cardinality function on $W$, by a weighted cardinality function, with weights~$\Delta$). Then, $F$ is submodular (as a consequence of the equivalence with the previously defined functions, which we now prove).
 
 These two types of functions are in fact equivalent. Indeed, for a weight function $D: 2^V \to \rb_+$, we let $W = 2^V$ and $S_{k} = \{ G \subset V, G \ni k \}$, and $\Delta(G) = { \rm Dep}(G)$, to obtain a set cover.
 
 For a certain set cover define by $W$, $S_k \subset W$, $k \in V$, and $\delta$, define 
 $${ \rm Dep}(G) = \sum_{j \in W} \Delta_j 1_{ G_j = \bigcup_{k \in V ,\  S_k \ni j} S_k },$$
  to obtain a set-function expressed in terms of groups and non-negative weight functions.

\paragraph{Examples.} In Figure~\ref{fig:tree}, we show a set of groups (i.e., only the groups $G \subset V$ for which ${ \rm Dep}(G)>0$), which can be embedded into a hierarchy, as well as the corresponding flow interpretation from \mysec{flows}. We also show in \myfig{1d} and \myfig{1dbis} examples in one dimension.

\begin{figure}

\begin{center}
 
\includegraphics[scale=.7]{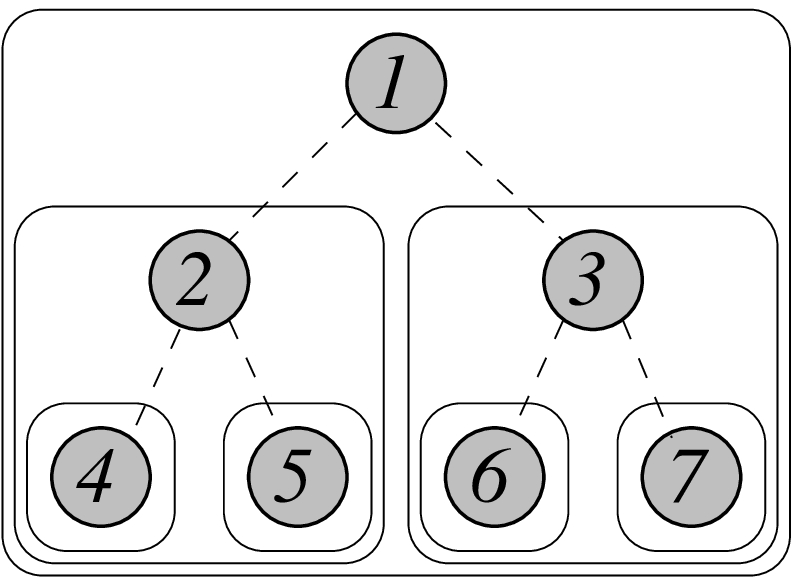} \hspace*{1cm}
\includegraphics[scale=.7]{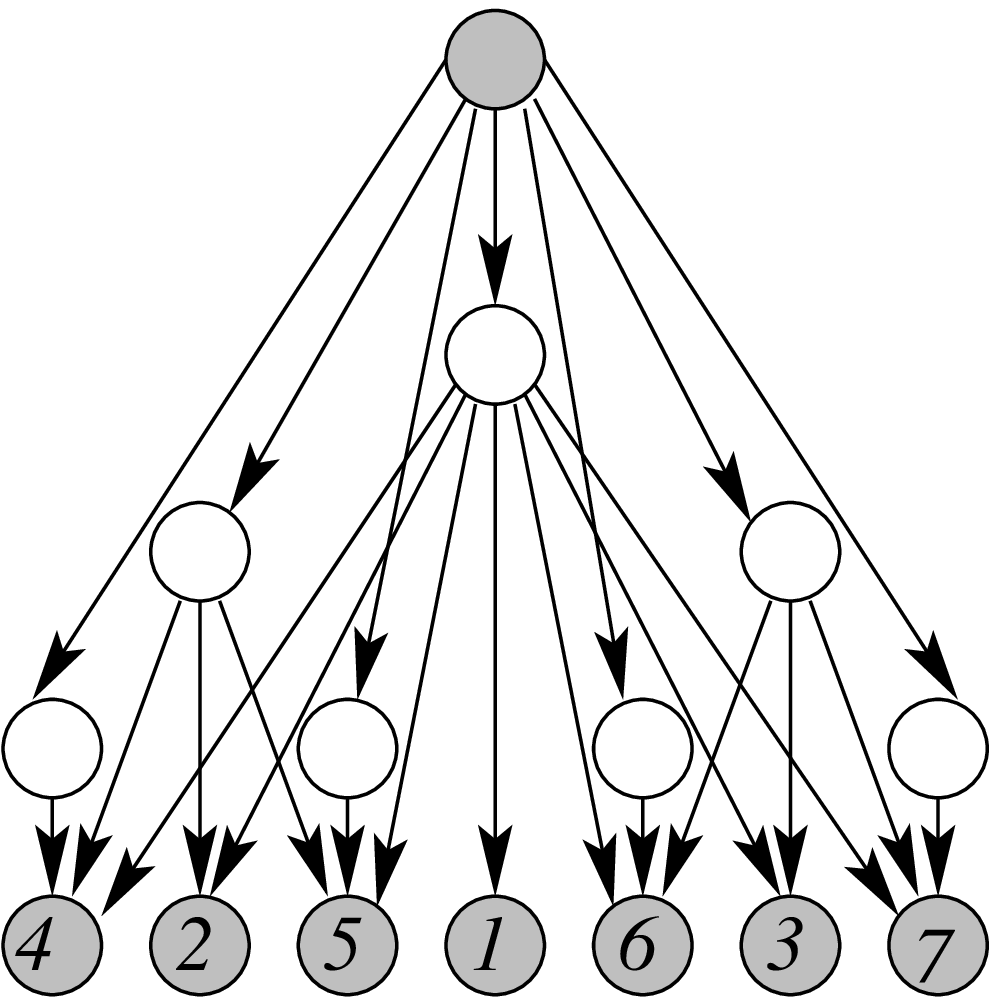}
 \end{center}

\caption{Left: Groups corresponding  to a hierarchy. Right: network flow interpretation of same submodular function.}
\label{fig:tree}
\end{figure}

\subsection{Flows}
\label{sec:flows}
Following~\cite{megiddo1974optimal}, we can obtain a family of non-decreasing submodular set-functions (which include set covers) from multi-sink multi-source networks. We define a weight function on a set $W$, which includes a set $S$ of sources and a set $V$ of sinks (which will be the set on which the submodular function will be defined). We assume that we are given capacities, i.e., a function $c$ from $W \times W$ to $\rb_+$. For all functions $\varphi: W \times W \to \rb$, we use the notation $\varphi(A,B) = \sum_{k \in A, \ j \in B} \varphi(k,j)$.

A flow is a function $\varphi: W \times W \to \rb_+$ such that (a) $\varphi \leqslant c$ for all arcs, (b) for all $w \in W \backslash (S \cup V )$, the net-flow at $w$, i.e., $\varphi(W,\{w\}) - \varphi( \{w\},W)$, is null, (c) for all sources $s \in S$,  the net-flow at $s$ is non-positive, i.e., $\varphi(W,\{s\}) - \varphi( \{s\},W) \leqslant 0$, (d) for all sinks $t \in V$, 
the net-flow at $t$ is non-negative, i.e., $\varphi(W,\{t\}) - \varphi( \{t\},W) \geqslant 0$.
We denote by $\mathcal{F}$ the set of flows.

For $A \subset V$ (the set of sinks), we define 
$$ \displaystyle F(A) = \max_{ \varphi \in \mathcal{F}} \ \varphi(W,A) - \varphi(A,W),$$
which is the maximal net-flow getting out of $A$. From the max-flow/min-cut theorem (see, e.g.,~\cite{cormen89introduction}), we have immediately that 
$$
F(A) = \min_{X \in W, \ S \subset X, \ A \subset W \backslash X } c(X, W \backslash X).
$$

One then obtain that $F$ is submodular (as the partial minimization of a cut function) and non-decreasing by construction. One particularity is that for this type of submodular  non-decreasing functions, we have an explicit description of the positive submodular polyhedron. Indeed, $x \in \rb_+^p$ belongs to $P(F)$ if and only if, there exists a flow $\varphi \in \mathcal{F}$ such that for all $k \in V$, $x_k = \varphi(W,\{k\}) - \varphi( \{k\},W)$ is the net-flow getting out of $k$.

Similarly to other cut-derived functions, there are dedicated algorithms for proximal methods and submodular minimization~\cite{hochbaum1995strongly}. See also~\cite{mairal10} for  applications to sparsity-inducing norms.

\paragraph{Flow interpretation of set-covers.} Following~\cite{mairal10}, we now show that the submodular functions defined in this section includes
the ones defined in \mysec{cover}. Indeed, consider a  {non-negative} function $D: 2^V \to \rb_+$, and define $F(A) = \sum_{G \subset V, G  \cap A \neq \varnothing} { \rm Dep}(G)$. The \lova extension may be written as, for all $w \in \rb_+^p$,
\BEAS
f(w) & = &  \sum_{G \subset V} { \rm Dep}(G) \max_{k \in G} w_k
\\
& = &  \sum_{G \subset V} \ \ \max_{ t^G \in \rb_+^p, \ t^G_{V \backslash G}=0, \ t^G(G) =  { \rm Dep}(G)}   w^\top t^G \\
& = &   \max_{ t^G \in \rb_+^p, \ t^G_{V \backslash G}=0, \ t^G(G) =  { \rm Dep}(G)}   \sum_{G \subset V}  w^\top t^G \\
& = &   \max_{ t^G \in \rb_+^p, \ t^G_{V \backslash G}=0, \ t^G(G) =  { \rm Dep}(G)}   \sum_{ k \in V} \bigg(  \sum_{G \subset V}  t^G_k \bigg) w_k .
\EEAS
Thus $s \in P(F)$, if and only there exists $t^G \in \rb_+^p, \ t^G_{V \backslash G}=0, \ t^G(G) =  { \rm Dep}(G)$ for all $G \subset V$, such that $ s = \sum_{G \subset V} t^G$. This can be given a network flow interpretation on the graph composed of a single source $s$, one node per subset $G \subset V $ such that ${ \rm Dep}(G)>0$, and the sink set $V$. The source is connected to all subsets $G$, with capacity ${ \rm Dep}(G)$, and each subset is connected to the variables it contains, with infinite capacity. We give examples of such networks in \myfig{1d} and \myfig{1dbis}.

\begin{figure}

\begin{center}
 \includegraphics[scale=.7]{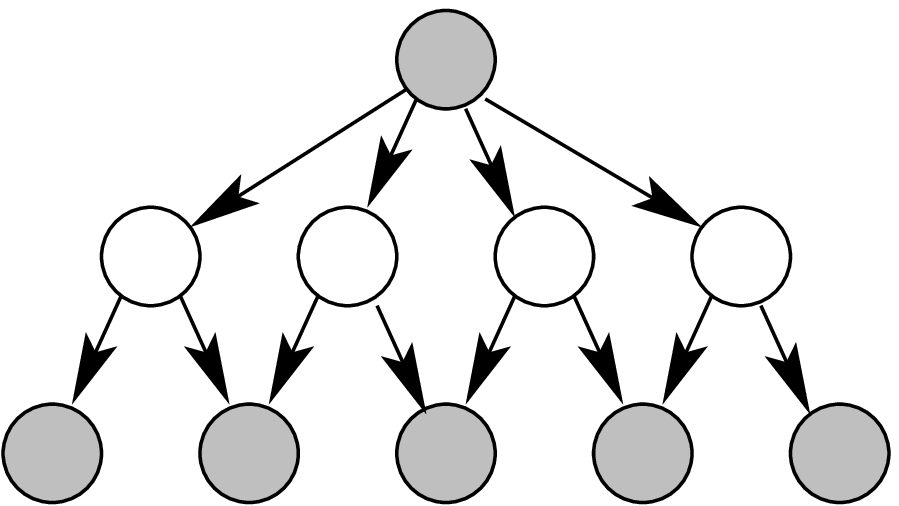} \hspace{1cm}
\includegraphics[scale=.7]{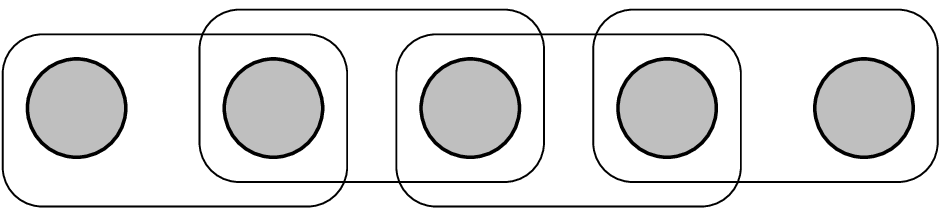}
  \end{center}

\caption{Flow (left) and set of groups (right).}

\label{fig:1d}

\end{figure}

\begin{figure}

\begin{center}
 \includegraphics[scale=.7]{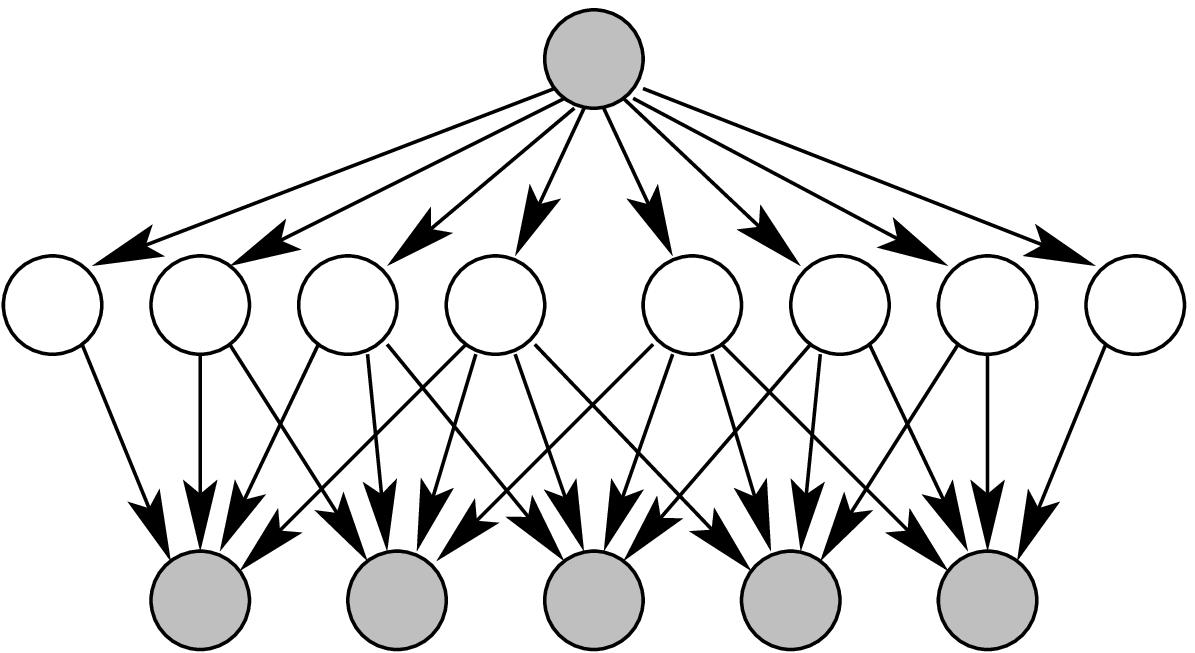} \hspace{1cm}
\includegraphics[scale=.7]{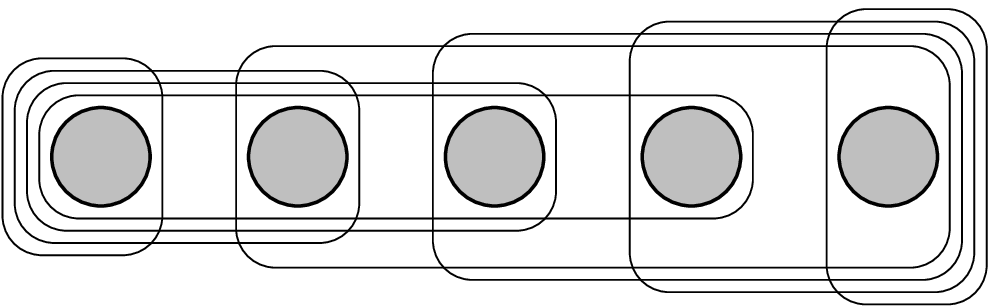}
 \end{center}

\caption{Flow (top) and set of groups (bottom).}

\label{fig:1dbis}

\end{figure}

\subsection{Entropies}
Given $p$ random variables $X_1,\dots,X_p$ which all take a finite number of values, we define $F(A)$ as the joint entropy of the variables $(X_k)_{k \in A}$. This function is submodular because, if $A \subset B$ and $k \notin B$,
$F(A \cup \{ k\}) - F(A) = H(X_A,X_k)-H(X_A) = H(X_k|X_A)  \geqslant  H(X_k|X_B) =   F(B \cup \{ k\}) - F(B)
$ (by the data processing inequality~\cite{cover91elements}).

This can be extended to any distribution by considering differential entropies. One application is for Gaussian random variables, leading to the submodularity of the function defined through $F(A) = \log \det Q_{AA}$, for some positive definite matrix $Q \in \rb^{p \times p}$ (see further related examples in 
\mysec{spectral}).

\subsection{Spectral functions of submatrices}
\label{sec:spectral}

Given a positive semidefinite matrix $Q \in \rb^{ p \times p }$ and a real-valued  function $h$ from $\rb_+ \to \rb$, one may define $\tr [h(Q)]$ as $\sum_{i=1}^p h(\lambda_i)$ where $\lambda_1,\dots,\lambda_p$ are the (nonnegative) eigenvalues of $Q$~\cite{horn1990matrix}. We can thus define the function $F(A) = \tr h(Q_{AA})$ for $A \subset V$.

The concavity of $h$ is not sufficient for submodularity (as can be seen by generating random examples with $h(\lambda) = \lambda/( \lambda+1)$).

We know however that the functions $h(\lambda) = \log( \lambda + t)$ for $t\geqslant 0$ lead to submodular functions; thus, since for $p \in (0,1)$,
$\lambda^p = \frac{ p \sin p \pi }{\pi} \int_0^\infty  \log (1+ \lambda/t) t^{p-1} dt$~(see, e.g.,~\cite{ando1979concavity}), $h(\lambda) = \lambda^p$ for $p \in (0,1]$ are positive linear combinations of functions that lead to non-decreasing submodular set-functions.  We thus obtain a
non-decreasing submodular function. Applications may be found in~\cite{bach2010structured}.

This can be generalized to functions of the singular values of $X(A,B)$ where $X$ is a rectangular matrix, by considering the fact that singular values of a matrix $X$ are related to the eigenvalues of
$\left( \begin{array}{cc} 0 & X \\ X^\top & 0 \end{array} \right)$ (see, e.g.,~\cite{golub83matrix}).

\subsection{Best subset selection}
Following~\cite{das2008algorithms}, we consider $p$ random variables (covariates) $X_1,\dots,X_p$, and a random response $Y$ with unit variance, i.e., $\var (Y)=1$. We consider predicting $Y$ linearly from $X$. We consider
$F(A) = \var(Y|X_A)$. The function $F$ is a non-increasing function.

A variable $X_j$ is a suppressor for variable $X_i$, if $|{\rm Corr}(Y,X_i|X_j)| > | {\rm Corr}(Y,X_i)| $. Following~\cite{das2008algorithms}, we assume that there are no suppressor variables given any set $A$, i.e., we assume that
for all $A \subset V$, $i,j \notin A$,
$$|{\rm Corr}(Y,X_i|X_j,X_A)| \leqslant | {\rm Corr}(Y,X_i|X_A) |, $$

We then have:
$$
\var(Y|X_A,X_k) - \var(Y|X_A) = - {\rm Corr}( Y, X_k | X_A)^2,
$$
$$
\var(Y|X_A,X_j,X_k) - \var(Y|X_A,X_j) = - {\rm Corr}( Y, X_k | X_A,X_j)^2.
$$

This implies that $F$ is supermodular. Note however that the condition on suppressors is rather strong.

\subsection{Matroids}
Given a set $V$, we consider a family $\mathcal{I}$ of subsets of $V$ such that
(a) $\varnothing \in \mathcal{I}$, (b) $I_1 \subset I_2 \in \mathcal{I} \Rightarrow I_1 \in \mathcal{I}$, and (c) for all $I_1,I_2 \in \mathcal{I}$, $|I_1| < |I_2| \Rightarrow \exists k \in I_2 \backslash I_1, \ I_1 \cup \{k \} \in \mathcal{I}$. The pair $(V, \mathcal{I})$ is then referred to as a matroid, with $\mathcal{I}$ its family of independent sets. Then the rank function of the matroid, defined as
$\rho(A) = \max_{ I \subset A, \ A \in \mathcal{I} } |I| $, is submodular.

The classical example is the \emph{graphic matroid}; it corresponds to $V$ being an edge set of a certain graph, and $\mathcal{I}$ being the set of subsets of edges which do not contain any cycle. The rank function $\rho(A)$ is then equal to $p$ minus the number of connected components of the subgraph induced by $A$.

The other one is the \emph{linear matroid}. Given a matrix $M$ with $p$ columns, then a set $I$ is independent if and only if the set of columns indexed by $I$ is independent. The rank function $\rho(A)$ is then the rank of the columns indexed by $A$ (this is also an instance of functions from \mysec{spectral}).

\subsection*{Acknowledgements}

This tutorial was partially supported by grants from the Agence Nationale de la Recherche (MGA
Project) and from the European Research Council (SIERRA Project). The author would like to
thank Rodolphe Jenatton, Armand Joulin, Julien Mairal and Guillaume Obozinski for discussions related to submodular functions.
\bibliographystyle{unsrt}
\bibliography{submodular}

\end{document}